%% file: main.tex
\documentclass[dvipsnames,format=sigconf,anonymous=false,review=false]{acmart}
\AtBeginDocument{%
  }

\copyrightyear{2026}
\acmYear{2026}
\setcopyright{cc}
\setcctype{by}
\acmConference[GECCO '26]{Genetic and Evolutionary Computation Conference}{July 13--17, 2026}{San Jose, Costa Rica}
\acmBooktitle{Genetic and Evolutionary Computation Conference (GECCO '26), July 13--17, 2026, San Jose, Costa Rica}
\acmDOI{10.1145/3795095.3805057}
\acmISBN{979-8-4007-2487-9/2026/07}

\input{_used_packages.tex}

\input{_macros.tex}

\allowdisplaybreaks

\begin{document}

\title{Gray-Box Optimization and the Vertex Coloring Problem}

\author{Johanna Gasse}
\email{johanna.gasse@student.hpi.de}
\orcid{0009-0009-5664-3519}
\affiliation{%
  \institution{Hasso Plattner Institute\\ University of Potsdam}
  \city{Potsdam}
  \country{Germany}
}

\author{Antonia Heinen}
\email{antonia.heinen@student.hpi.de}
\orcid{0009-0002-4365-412X}
\affiliation{%
  \institution{Hasso Plattner Institute\\ University of Potsdam}
  \city{Potsdam}
  \country{Germany}
}

\author{Hendrik Higl}
\email{hendrik.higl@student.hpi.de}
\orcid{0009-0009-7867-0542}
\affiliation{%
  \institution{Hasso Plattner Institute\\ University of Potsdam}
  \city{Potsdam}
  \country{Germany}
}

\author{Timo Kötzing}
\email{timo.koetzing@hpi.de}
\orcid{0000-0002-1028-5228}
\affiliation{%
  \institution{Hasso Plattner Institute\\ University of Potsdam}
  \city{Potsdam}
  \country{Germany}
}


\renewcommand{\shortauthors}{Gasse et al.}
\begin{abstract}
Gray-box optimization is an approach for making some problem-specific information available to the algorithm while still relying on fitness information as the main guide to an optimum. This approach was shown to be beneficial in various combinatorial optimization tasks and neatly captures the continuum between fully black-box algorithms and tailored algorithms.

In this work, we discuss different flavors of gray-box algorithms. We show that RLS can find a proper $2$-coloring in a bipartite graph starting from a random $2$-coloring, in an expected time of $\mathcal{O}(n \log n)$. In contrast, when starting from a proper $n$-coloring, the (1+1) EA cannot find such a coloring except when offered additional guiding on plateaus of the search space. Finally, we show the run time for this setting can be much improved by using gray-box operators.
\end{abstract}




\maketitle

\input{sections/intro.tex}

\input{sections/prelim}

\input{sections/rls}

\input{sections/ooea}

\input{sections/ooea_bff}

\input{sections/graybox}

\section{Conclusion}

In this paper we discussed how different algorithms use problem knowledge to different degrees: Black-box algorithms are not tailored to the problem or instance at hand; problem-aware algorithms incorporate problem knowledge, but work the same on different instances of this problem, while instance-aware algorithms exploit specifics of the given instance, such as the neighborhood information of a graph.

We then first consider the conflict minimization problem on graphs, a $2$-coloring problem, for which we show that the standard RLS can solve it optimally on complete bipartite graphs.

Second, we consider using only proper colorings for finding an optimal coloring on complete bipartite graphs and on paths. As we have shown, progressively more information about problem and instance gains us run time asymptotically.

\bibliographystyle{ACM-Reference-Format}
\bibliography{_ref.bib}

\input{appendix.tex}

\end{document}

%% file: _used_packages.tex
\usepackage{hyperref}[hyperfootnotes=true]
\usepackage{float}
\usepackage{physics}
\usepackage{tabularx}
    \newcolumntype{Y}{>{\centering\arraybackslash}X}
    \renewcommand{\arraystretch}{1.5}
    
\usepackage{booktabs}
\usepackage{pdflscape}
\usepackage{csvsimple}
\usepackage{url}
\usepackage{nameref}
\usepackage{tablefootnote}
\usepackage{color, colortbl}
\usepackage{mathtools}
\usepackage{amsmath,amsfonts,mathtools}
\usepackage{bbm}
\usepackage[vlined,linesnumbered, ruled]{algorithm2e}
\usepackage{xparse}
\usepackage{algpseudocode}
\usepackage{makecell,booktabs}
\usepackage{subfigure}
\usepackage{caption}
\newcommand{\ignore}[1]{}
\usepackage{thmtools,thm-restate}
\usepackage[createShortEnv]{proof-at-the-end}
\usepackage{enumitem}

\usepackage{tikz} 
\usepackage{pgfplots}
\usepgfplotslibrary{fillbetween,statistics}
\pgfplotsset{compat=newest}
\pgfplotscreateplotcyclelist{tikzcycle}{%
draw=none \\ 
draw=none \\
fill=black!10 \\
thick,black \\ 
thick,blue,mark=square*\\ 
thick,red,mark=*\\ 
}
\pgfplotscreateplotcyclelist{muscycle}{%
thick,blue,mark=square*\\ 
thick,red,mark=*\\ 
}

\usepackage
[
    noabbrev,   
    nameinlink, 
]
{cleveref} 

%% file: _macros.tex
\DeclareDocumentCommand{\set}{m g o}{
    \ensuremath{
        \IfNoValueTF{#3}{\left}{#3}\{#1
            \IfNoValueTF{#2}{}{
                \ \IfNoValueTF{#3}{\left}{#3}\vert\ \vphantom{#1}#2\IfNoValueTF{#3}{\right.}{}
            } \IfNoValueTF{#3}{\right}{#3}\}
    }\xspace
}

\definecolor{lightgreen}{rgb}{0.75,0.92,0.61}

\newcommand{\labelname}[1]{
  \def\@currentlabelname{#1}}

\newcommand{\N}{\mathbb{N}}
\newcommand{\R}{\mathbb{R}}

\definecolor{mygreen}{RGB}{1, 150, 122}

\providecommand{\ignore}[1]{} 

\newcommand{\ooea}{(1+1) EA\xspace}

\newcommand{\col}{\mathrm{col}\xspace}

\newcommand{\Ber}{\mathrm{Ber}}

\newcommand{\Prob}[1]{\mathbb{P}\left[ #1 \right]}
\newcommand{\Ex}[1]{\mathbb{E}\left[ #1 \right]}
\newcommand{\Var}[1]{\mathbb{V}\left[ #1 \right]}

\newcommand{\bigmid}{~\big|~}

\DeclareMathOperator{\Geo}{Geo}

\newcommand{\natnum}{\mathbb{N}}

	{
	\begin{center}
	\begin{algorithm2e}
	}%
	{
	\end{algorithm2e}
	\end{center}
	}

\allowdisplaybreaks

%% file: sections/intro.tex
\newcommand{\figureOfResults}{%
\begin{figure*}
    \bgroup
    \def\arraystretch{1.5}
    \setlength{\tabcolsep}{10pt}
    \begin{tabular}{cccc}
        \textbf{Fitness function} & \textbf{Algorithm} & \textbf{Complete Bipartite Graphs} & \textbf{Paths} \\ \hline
        $\textsc{NumUsedColors}$ & \ooea & Exponential, Thm \ref{thm:ooea_bipartite_lowerbound} & Exponential, Thm \ref{thm:ooea_bounded_degree} \\
        $\textsc{RankedColors}$ & \ooea & $\mathcal{O}(n^{3} \log \log n)$, Thm \ref{the:run time_bff_kmn}; $\Omega(n^{2})$, Thm \ref{thm:ooea_bbff_kmn_lower} & $\mathcal{O}(n^{7})$, Thms \ref{the:run time_bff_kmn2} \& \ref{the:3_to_2_bff}; $\Omega(n^6)$, Thm \ref{thm:ooea_bff_paths_lowerbound} \\
        $\textsc{NumUsedColors}$ & Graybox RLS & $\Theta(n \log n)$, Thms \ref{the:gb_on_knm} \& \ref{the:lower_bound_gbkmn} & N/A \\
        $\textsc{RankedColors}$ & Graybox RLS & $\Theta(n \log n)$, Thms \ref{cor:bgg_kmn_upper_bound} \& \ref{the:lower_bound_gbkmn} & $\mathcal{O}(n^{4})$, Thms \ref{the:to_deg+1} \& \ref{the:3_to_2}
    \end{tabular}
    \egroup
    \captionsetup{type=table}
    \caption{Run time results for various algorithms and fitness functions for the minimal coloring problem. All bounds are on the expected number of iterations, except for the exponential bound, which holds with high probability. 
    }
    \label{tab:results}
\end{figure*}}

\section{Introduction} \label{sec:introduction}

Evolutionary algorithms can be used as black-box heuristics, oblivious to specific problem characteristics and purely relying on querying quality (the fitness) of potential solutions as guiding information. While this approach has proven to lead to surprisingly good results in many areas, better results are typically gained when opening up the black box just a little to incorporate specific knowledge about the problem.

Championing this approach is Whitley in his various works on NK-landscapes, including, for example, Max-$k$-SAT (see \cite{whitley_mk_2015}). Here, the fitness function is composed of sub-functions relying on at most $k$ problem variables each. For small $k$, this additional problem knowledge can be utilized to improve performance of the solver, at the price of generality (applicability to other problems).

An alternative approach keeps the setting of evaluating the fitness function in a black-box manner, but using problem knowledge to design specific variation operators. This was used in practice in \cite{peters_mixed_2019}, where a real-world scheduling problem was solved by employing various application-specific mutation operators.

Theoretical analyses have shown that such operators tailored to the problem at hand can lead to asymptotically faster expected run times. In \cite{baguley_analysis_2022} the authors showed, for the combinatorial vertex cover problem, that an operator specific for the vertex cover \emph{and aware of the instance to be solved} can improve the run time from $\Theta(n^4)$ to $\Theta(n^3)$ for some instances of the vertex cover problem.

\figureOfResults

\subsection{What is Gray-Box?}
\label{sec:whatIsGraybox}

We argue that these instance-aware operators are just the tip of the iceberg regarding gray box optimization: Ranging from constraint handling to balanced-flip operators, opening the black box is not the exception, but the norm.

As our first example, consider infeasible regions of the search space. Assigning infeasible search points a fixed bad fitness is a straightforward (and truly black-box) approach. However, this results in no guiding information for algorithms stuck in infeasible regions (for example at an already infeasible starting point). A typical way to mitigate this problem is to have a two-part fitness function: first minimize how much the constraints are violated; second optimize the actual cost function (see, for example, \cite{back_independent_1994,neumann_randomized_2007}, where penalty functions were used for combinatorial problems). Note that these penalty functions are typically folded into a unary objective function, but a binary objective with lexicographic ordering is the cleaner solution.

This example can be generalized: A fitness function can be \emph{augmented} to provide guidance on a plateau in order to cross areas without fitness gradient more effectively. We will use such an approach in this paper.

One can also imagine actually \emph{deforming} the fitness function to provide better guidance. In the most extreme case, the fitness function gives the distance to the (closest) optimum and thus provides perfect information, since now simple hill-climbing can solve this OneMax-like problem. However, this means that the fitness function needs to solve the optimization problem, which is the problem we started out with. Thus, considering such deformations is in the spirit of gray-box optimization, with the natural requirement that the fitness function is efficiently computable in the instance.

Regarding variation operators, we already mentioned \emph{instance-aware} operators. For example, in \cite{baguley_analysis_2022}, a mutation operator for the vertex cover problem can choose to remove a vertex from the cover, but instead add a neighbor to the cover. This operator thus requires knowledge of the problem \emph{instance} in order to determine which vertices are neighbors. Another example for a gray-box operator for vertex cover was given in \cite{he_comparative_2005}. Some operators are even more intrusive and employ specialized repair mechanisms, see, for example, \cite{branson_focused_2021}. In this paper we also give instance-aware operators.

While instance-aware operators open up the black box a lot, some other operators are much less intrusive, but still not fully unbiased. Consider the knapsack problem. For this problem we know that (typically) adding an item is better, unless constraints are violated. This is independent of concrete instances, but still variation operators can make use of this knowledge. For example, a ``balanced flip'' can remove one item and add some other item (similarly for crossover operators). The work of \cite{friedrich_crossover_2023} shows that such \emph{problem-aware} operators, which tailor the algorithm to the problem at hand but not to the concrete instance to be solved, can asymptotically improve the expected run time for certain classes of problem instances.
For both problem- and instance-aware operators, we naturally also require that they should be efficiently computable.

\subsection{Graph Coloring}

In this paper we consider the classic vertex coloring problem on unweighted simple graphs. Given such a graph on $n$ vertices, the goal is to assign each vertex a color such that no two neighboring vertices have the same color assigned (we call an edge between two vertices of the same color \emph{monochromatic}). Using $[n] = \{1,\ldots,n\}$, we code colors as natural numbers and thus work in the domain $[n]^n$, since we need at most $n$ colors for a graph with $n$ vertices. Note that, in spite of the coding as numbers, our variables are categorical and not numerical, so many algorithms and analyses for multi-valued decision variables (see, for example, \cite{doerr_static_2018}) do not apply.

Analyses of evolutionary algorithms for graph coloring were considered before in the literature. While \cite{sudholt_crossover_2005} proves that the standard \ooea takes exponential time to 2-color complete bipartite graphs, \cite{cheong_analysis_2010} develops an Iterated Local Search Algorithm, which is able to find an optimal coloring on bipartite graphs in $\mathcal{O}(n^{2} \log n)$. It does so by using an instance-aware gray-box mutation operator called ``Color Elimination'' (based on Kempe chains), which aims to remove a color $i$ from the neighborhood of a vertex~$v$. Using an instance-aware gray-box operator directly based on Kempe chains, it also gives polynomial-time results for sparse random graphs, odd rings and 5-coloring general planar graphs. This work is extended by \cite{bossek_time_2021} for the dynamic graph coloring problem, which finds that mutation operators tailored to the portion of the graph where changes have occurred can reduce the re-coloring time. 

Graph coloring has also been considered on bipartite graphs for instances that start with a random $2$-coloring and then seek to minimize the number of monochromatic edges in the coloring. \cite{fischer_one-dimensional_2005} analyzes the performance of both RLS and the \ooea on the Ising Model, which is known to be equivalent to $2$-coloring on bipartite graphs, on rings. They arrive at a run time bound of $\mathcal{O}(n^3)$ fitness function evaluations for both on rings with $n$ vertices. As the present paper proves, the \ooea starting at $n$ colors takes exponential time to find a $2$-coloring on paths, showing that coloring starting at $n$ colors can be harder to solve than the same problem starting at $2$ colors.

\subsection{Our Contribution}

In this paper, we analyze randomized search heuristics with and without gray-box operators on the graph coloring problem. The straightforward fitness function counts the number of colors used; we deal with infeasible solutions by first minimizing the number of monochromatic edges in a lexicographic way. We call this fitness function $\textsc{NumUsedColors}$.

Even on the feasible solutions, this results in many plateaus without fitness gradient as guiding information: Consider the graph on $n$ vertices with no edges, where half the vertices are colored green and the other red. Now the only way to gain fitness is to have all vertices colored in the same color. We augment the fitness function by having as a secondary objective to minimize the number of vertices of the least-used color, as tertiary objective to minimize the number of vertices of the second-least-used color, and so on. We call this fitness function $\textsc{RankedColors}$. Intuitively, in order to remove a color, we guide the algorithm towards minimizing the use of the least-used colors. See Section~\ref{sec:prelim_rls} for more details.

We consider the two classic algorithms random local search (RLS) and the (1+1) Evolutionary algorithm (\ooea). For RLS, we also consider a gray-box operator which gives preference to removing less used colors and does not attempt relabeling with unused colors. It also attempts swapping colors with neighbors. Note that the preference is problem-aware, as is the focus on unused colors, but not instance-aware. On the other hand, swapping colors with neighbors (or avoiding colors of neighbors) is instance-aware. See, again, Section~\ref{sec:prelim_rls} for details on this operator.

As a side note, we consider the related problem of minimizing monochromatic edges of a $2$-coloring. Note that, for bipartite graphs, the global optimum will then be a proper $2$-coloring of the graph. In Section~\ref{sec:rls2col} we show that, for this classic search space and the plain RLS on complete bipartite graphs, just the guiding information about the number of monochromatic edges provided by $\textsc{NumUsedColors}$ is sufficient for efficient optimization in $\mathcal{O}(n \log n)$ fitness evaluations.

In Section~\ref{sec:ooea} we turn to the \ooea with unbiased mutation. We change the search space to the $[n]^n$ and want to minimize the number of colors used. We show that this algorithm, with high probability, has exponential run time to find an optimal coloring for the two bipartite graph classes of complete bipartite graphs and paths when using the fitness function $\textsc{NumUsedColors}$.

By contrast, in Section~\ref{sec:ooea_bff} we note that the \ooea can optimize in expected polynomial time (number of function evaluations) for both mentioned graph classes when using the augmented fitness function $\textsc{RankedColors}$.

Finally, in Section~\ref{sec:graybox} we show that RLS with our gray-box operator can significantly speed up optimization, both for complete bipartite graphs (where already the fitness function $\textsc{NumUsedColors}$ leads to an expected optimization time of $\Theta(n \log n)$), and on paths (where our analysis makes used of $\textsc{RankedColors}$ in order to bring down the expected optimization time to $\mathcal{O}(n^{4})$). We give an overview over our results on the coloring problem in Table~\ref{tab:results}. 

Before diving into the contribution section, we give mathematical preliminaries, including algorithms and fitness functions, in Section~\ref{sec:prelim}.

%% file: sections/prelim.tex
\section{Objectives and Algorithms} \label{sec:prelim}

In this section, we introduce general notation as well as the specific problems and algorithms considered in this paper.

For $n \in \N$ we write $[n] = \{ 1, \dots, n \}$.

Let $A, B$ be sets that are not necessarily disjoint. Then we call $A \uplus B$ the disjoint union of $A$ and $B$ and assume any method of making them disjoint. Note that, for ease of notation, we still refer to elements of $A$ and $B$ as elements of $A \uplus B$.

\begin{definition}[Vertex Coloring]
    Let $G = (V, E)$ be a graph, and $k \in \N$. We call a function $c \colon V \to [k]$ a $k$-coloring of $G$. For a vertex $v \in V$, we call $c(v)$ its color.
    Let $\{ u, v \} \in E$ be an edge. If $c(u) = c(v)$, then $\{ u, v \}$ is called a monochromatic edge. If $c$ has no monochromatic edges, we call it a proper coloring. If a proper $k$-coloring of $G$ exists, we call $G$ $k$-colorable.
\end{definition}

While we define a coloring as a function mapping vertices to colors, in the remainder of this paper we assume an order of vertices $v_1, \ldots, v_n$ and use an array $x \in [n]^n$ for colorings, where, for $i \in [n]$, $x[i] = c(v_i)$.

For $M,N \in \natnum$, we consider complete bipartite graphs $K_{M,N} = ([M] \uplus [N], [M] \times [N])$ with $M$ and $N$ vertices in the two partitions, respectively. In this context, $n$ refers to the total number of vertices with $n = M + N$. For $n \in \natnum$, we also consider paths on $n$ vertices $P_n = ([n], \{ \{ i, i + 1 \} \mid i \in [n - 1]\})$.

\subsection{Objectives} \label{sec:prelim_rls}

We introduce two fitness functions. Both will work on tuples, using the following lexicographical ordering. For $n \in \N$ and $x, y \in \N^n$ we write $x < y$ if there is $i \in [n]$ such that $x[i] < y[i]$ and, for all $j < i$, $x[j] = y[j]$. That is, $x$ is less than $y$ in the first entry where they disagree.   We write $x \leq y$ if $x = y$ or $x < y$. This ordering relation is trivially total.

We let $\textsc{NumUsedColors} \colon [n]^n \to \N^{2}$ be the following fitness function: For a given coloring $x \in [n]^n$, let $\textsc{NumUsedColors}(x)[1]$ be the number of monochromatic edges, given by $\lvert \{\{ u, v \} \in E \mid x[u] = x[v]\} \rvert$. Let $k \in \N$ be the total number of colors that appear in $x$. Then, let $\textsc{NumUsedColors}(x)[2]=k$.

This fitness function's primary objective is minimizing the number of monochromatic edges in the graph. As a secondary objective, it minimizes the total number of colors used in the graph. Therefore, once the algorithm discovers a proper coloring, the fitness function never accepts an improper coloring afterwards, but prefers a smaller number of colors used.

We also analyze an alternative fitness function, which we denote by $\textsc{RankedColors} \colon [n]^n \to \N^{n + 1}$. We first define the term \emph{rank}.

\begin{definition}[Rank] \label{def:rank}
    Let $G = (V, E)$ be a graph and $x \in [n]^n$ a coloring of $G$. Let $a$ be an ordering of all colors in $x$ that appear at least once, ordered ascending by frequency. For a color $c \in [n]$, we call its $1$-based index in $a$ its rank, denoted by $r_c$.
\end{definition}

For a given coloring $x \in [n]^n$, let $\textsc{RankedColors}(x)[1]$ be the number of monochromatic edges. Let $k \in \N$ be the total number of colors that appear at least once in $x$. Then, for all $2 \leq i \leq n - k + 1$, let $\textsc{RankedColors}(x)[i] = 0$. For all $i \in [k]$, let $\textsc{RankedColors}(x)[{n - k + i + 1}]$ be the total number of vertices colored in the color of rank $i$.

\begin{example}
    Let a graph with $3$ vertices with a coloring $x = (1, 2, 2)$ be given. Here, color $3$ does not appear at all, color $1$ has rank $1$ and color $2$ has rank $2$. Further assume that there are no monochromatic edges. Then $\textsc{RankedColors}(x) = (0, 0, 1, 2)$.

    Here, the first $0$ represents the absence of monochromatic edges. The second $0$ represents the zero vertices colored in color $3$, which is absent in $x$. Then there is one vertex colored in the color of rank $1$ and two vertices colored in the color of rank $2$.
\end{example}

Like with $\textsc{NumUsedColors}$, any coloring without monochromatic edges is better than any coloring with monochromatic edges when evaluated with~$\textsc{RankedColors}$. The latter also prioritizes reducing the number of appearances of less frequent colors further, while maintaining that the number of colors in use in the graph can never increase. This means that $\textsc{RankedColors}$ will reject any changes that increase the number of appearances of one color without decreasing the number of appearances of a less frequent color. In particular, an optimal coloring in this ordering will be a coloring using a minimal number of colors.

\subsection{Algorithms} \label{sec:prelim_rls}

Let $n,\ell \in \N$ and $f\colon \natnum^n \to \N^{\ell}$. Given $\textsc{Mutate}\colon \N^n \to \N^n$ (called a mutation operator) and a (possibly randomized) value $\textsc{Start}$ (called starting individual), we consider the following generalized trajectory-based optimization algorithm. This algorithm starts with an initial and usually non-optimal coloring $x$ and iteratively attempts to improve it by mutating it. We give the formal definition for minimization of a fitness function $f$ in Algorithm~\ref{alg:tboa}.

\begin{algorithm}
    let $x = \textsc{Start}$\;
    \While{true}
    {
        let $x' = \textsc{Mutate}(x)$\;
        \If{$f(x') \leq f(x)$}
        {
            set $x \leftarrow x'$\;
        }
    }
	\caption{Trajectory Based Optimization Algorithm}
    \label{alg:tboa}
\end{algorithm}

\noindent
\textbf{The Domain $\{0,1\}^n$.}
In this domain the starting individual is $x \in \{0, 1\}^{n}$ uniformly at random. We define the classic mutation operator $\textsc{OneBitFlip}$ such that, for all $x \in \{0,1\}^n$, $\textsc{OneBitFlip}(x)$ changes exactly one bit chosen uniformly at random from the input. Note that this operator is also known as the random local search (RLS) operator.

We refer to Algorithm~\ref{alg:tboa} with this uniform starting individual and the RLS operator as Random Local Search (RLS).

\noindent
\textbf{The Domain $[n]^n$.}
In this domain we can start with a coloring which is guaranteed to be proper: As starting individual we choose $x=(1,2,...,n)$. Note that, while this coloring is proper, in most cases it is non-minimal.

We consider the operator $\textsc{IndependentCategoricalFlip}$ such that, for all $x \in [n]^n$, $\textsc{IndependentCategoricalFlip}(x)$ makes changes in each position $j \in [n]$ independently with probability $1/n$ to a uniformly random $a \in [n]$. Note that this operator is an extension of the uniform bit flip operator employed by the classic \ooea.

For this domain, we have an additional operator, a gray-box operator, given in Algorithm~\ref{alg:graybox_mut}. It utilizes the following two operators, each accepting a coloring $x \in [n]^{n}$ and a color $c$ as parameters.

\begin{itemize}[leftmargin=*]
    \item $\textsc{Flip}(x, c)$: Choose a vertex $v$ colored in $c$ uniformly at random. Then choose a different color currently present in the coloring that is not adjacent to $v$ and recolor $v$ in this color, if possible.
    \item $\textsc{Swap}(x, c)$: Choose a vertex $v$ colored in $c$ and a neighbor $u$ of $v$ uniformly at random and swap the colors of $u$ and $v$.
\end{itemize}
The idea is to perform flips that we know are advantageous, that is, flips reducing less frequent colors, with higher probability. This is unlike the unbiased mutation operator used in the \ooea, which performs any flip with equal probability. There are, however, situations where no flip to a different color is possible. In such a situation, a swap might be necessary to lay the groundwork for a flip at a later point in time.

\SetKwIF{WithProb}{ElseWithProb}{WithProbElse}{with probability}{do}{else with probability}{else}{endWithProbability}

\SetKwBlock{WProbOneN}{with probability $\frac{1}{n}$ do}

\SetKwBlock{WProbOneTwo}{with probability $\frac{1}{2}$ do}{else}

\begin{algorithm}
    \textbf{Input:} $x \in [n]^n$\;
    choose $r \sim \Geo \left( 1 / 2 \right)$\;
    let $k$ be the total number of colors used in $x$\;
    \If{$r > k$}
    {
        \Return $x$\;
    }
    let $c$ be the color of rank $r$\;
    \uWithProb{$1/2$}
    {
        \Return $\textsc{Flip}(x, c)$\;
    }
    \WithProbElse
    {
        \Return $\textsc{Swap}(x, c)$\;
    }
    \caption{Gray-Box Mutation Operator}
    \label{alg:graybox_mut}
\end{algorithm}

As complete setups for the domain $[n]^n$, we consider the following two algorithms cases. Each uses the trajectory based optimization algorithm (Algorithm~\ref{alg:tboa}).
\begin{itemize}[leftmargin=*]
    \item Using the $\textsc{IndependentCategoricalFlip}$ operator and a proper coloring as a starting individual; we refer to this setting as \ooea with unbiased mutation.
    \item Using the gray-box operator and a proper coloring as starting individual; we refer to this setting as gray-box RLS.
\end{itemize}

%% file: sections/rls.tex
\section{RLS with Two Colors} \label{sec:rls2col}

In this section we consider the problem of conflict minimization (minimizing the number of monochromatic edges in a two-colored graph), and show that Random Local Search takes at most $ \mathcal{O}(n\log n)$ fitness evaluations in expectation to find a 2-coloring without monochromatic edges on complete bipartite graphs.

\begin{theorem} \label{thm:rls_runtime}
The run time for Random Local Search (RLS) with fitness function \textsc{NumUsedColors} using the search space $\{0,1\}^n$ and starting from a possibly non-proper 2-coloring to find a proper 2-coloring on $K_{M,N}$ is $\mathcal{O}(n\log n)$ in expectation.
\end{theorem}

We prove the theorem by dividing the optimization process into different steps. For this we characterize the states that are encountered by RLS according to the colorings of the two sets of the partition of the graph. If such a set has more vertices in one color than in the other, it has a \emph{preference} for this color. 
If the set has a preference, we refer to it as \emph{unbalanced}, and \emph{balanced} otherwise. If both partition sets have the same preference, the state is \emph{negatively unbalanced}, if either set is balanced, so is the state, and it is \emph{positively unbalanced} when the sets have different preferences. This leads to the optimal state, where both sets are fully colored in one color. We consider the states in this order, where the optimal state is the highest one.

We show that RLS in one such state only transitions to states that are higher or equal in this ordering. We also prove that RLS only spends, in expectation, at most $\mathcal{O}(n\log n)$ time in non-optimal states. Intuitively, the initial search point might already have a preference for white in one set and for black in the other. This corresponds to a positively unbalanced state, and optimization is a straightforward hill-climb. If the initial preference is for the same color, then the algorithm starts in a negatively unbalanced state. In this case, the fitness will guide the search such that both partitions become more and more balanced, until either the search point is positively unbalanced or balanced. We will show that, in a balanced state, there are neutral mutations (neither gaining nor losing in fitness). By a Markov chain analysis we show that, in time $\mathcal{O}(n)$ in expectation, the algorithm leaves the balanced state and enters a positively unbalanced state, from where simple hill-climbing will attain the optimum. For the full proof, see \Cref{sec:appendix}.

%% file: sections/ooea.tex
\section{\ooea with NumUsedColors} \label{sec:ooea}

In this chapter, we analyze the run time of the \ooea with fitness function \textsc{NumUsedColors} on complete bipartite graphs and paths and conclude that it is exponential with very high probability.

\subsection{Complete Bipartite Graphs}

On complete bipartite graphs, the \ooea with fitness function \textsc{NumUsedColors} takes exponential time with high probability. The problem is that, given any number of colors $k$, the algorithm will keep variously recoloring vertices. This is called a fitness plateau, as the fitness function does not distinguish between different colorings with $k$ colors.

\begin{theoremE}[][end, restate, category=unbiased] \label{thm:ooea_bipartite_lowerbound}
    Let $M, N \in \N$ with $n \coloneq M + N$. Let $T$ be the run time of the \ooea, using fitness function \textsc{NumUsedColors}, to $2$-color $K_{M,N}$.
    Then there exists a constant $c > 0$ such that
    \begin{align*}
        \Prob{T \leq 2^{cn}} \in 2^{-\Omega(n)}.
    \end{align*}
\end{theoremE}

\begin{proofE}
    Let $V_M, V_N \subseteq V$ be the parts of $G$ with sizes $M$ and $N$ respectively. Without loss of generality, assume $M \geq N$.

    For a given coloring, we call the most frequent color in $V_M$ \emph{good} and all remaining colors \emph{bad}. When there are multiple most frequent colors at once, any one of them can be considered good. We call any vertex colored in a good or bad color good or bad respectively. Then let $X_t$ denote the total number of bad vertices at time $t \in \N$ when using the \ooea.

    Let $T' = \min \{ t \in \N \mid X_t \leq 0 \}$ be the first moment in time when $V_M$ is colored in a single color. This is a necessary condition for $K_{M,N}$ to be $2$-colored. Therefore, a lower bound on $T'$ also imposes a lower bound on $T$, the time that it takes to $2$-color $K_{M,N}$ in total, using the \ooea and fitness function $\textsc{NumUsedColors}$.
    
    We analyze the process $(X_t)_{t \in \N}$ over the interval $\left[ 0, \frac{M}{4} \right]$ and need to show that it exhibits negative drift. As the process begins with each vertex colored in a unique color, $X_0 = n - 1 > \frac{M}{4}$ for sufficiently large $n$.
    
    Let $t \in \N$. Assume $0 < X_t < \frac{M}{4}$, and let $k$ be the total number of colors in $V_M$ at time $t$. As we are not yet finished, $k \geq 2$. Define $A$ as the event that the current mutation does not increase the total number of colors present in $V_M$. In the following, we condition on $A$. Furthermore, let $B$ denote the event that the current mutation only mutates a single vertex.

    \newcommand{\baddrift}{Y_\uparrow}
    \newcommand{\gooddrift}{Y_\downarrow}

    Let $\baddrift$ denote the total number of good vertices turning bad at time $t$. Analogously, let $\gooddrift$ denote the total number of bad vertices turning good at time $t$. We now proceed by proving a lower bound on $\Ex{\baddrift \bigmid X_t; A }$ and an upper bound on $\Ex{\gooddrift \bigmid X_t; A}$.

    Let $k$ be the total number of colors in $V_M$ at time $t$. Using the law of total expectation, we get
    \begin{align*}
        &\Ex{\baddrift \bigmid X_t; A} \\
        \geq~&\Ex{\baddrift \bigmid X_t; A, B} \cdot \Prob{B \bigmid X_t; A} \\
        =~&\frac{M - X_t}{n} \cdot \frac{k - 1}{k} \cdot n \cdot \frac{1}{n} \cdot \left(1 - \frac{1}{n}\right)^{n - 1} \\
        \geq~&\frac{M - X_t}{n} \cdot \frac{1}{2} \cdot \frac{1}{e}.
    \end{align*}

    For $\gooddrift$, we note that, out of $k$ total colors in $V_M$, only one is considered good. Therefore
    \begin{align*}
        \Ex{\gooddrift \bigmid X_t; A} \leq X_t \cdot \frac{1}{n} \cdot \frac{1}{k} \leq \frac{1}{2} \cdot \frac{X_t}{n}.
    \end{align*}

    To show that there is negative drift, consider $X_{t + 1} - X_t$. This is the total number of bad vertices gained minus the total number of good vertices gained, so \mbox{$X_{t + 1} - X_t = \baddrift - \gooddrift$}. Therefore,
    \begin{align*}
        &\Ex{X_{t + 1} - X_t \bigmid X_t; A} \\
        =~&\Ex{\baddrift \bigmid X_t; A} - \Ex{\gooddrift \bigmid X_t} \\
        \geq~&\frac{1}{2e} \cdot \frac{M - X_t}{n} - \frac{1}{2} \cdot \frac{X_t}{n} \\
        =~&\frac{1}{2e} \cdot \frac{1}{n} \cdot (M - (1 + e)X_t) \\
        \intertext{Using $X_t \leq \frac{M}{4}$ and $M \geq \frac{n}{2}$, we get}
        \geq~&\frac{1}{2e} \cdot \frac{1}{n} \cdot \left( M - (1 + e) \cdot \frac{M}{4} \right) \\
        =~&\frac{1}{8e} \cdot \frac{1}{n} \cdot (4M - (1 + e)\cdot M) \\
        =~&\frac{3 - e}{8e} \cdot \frac{1}{n} \cdot M \\
        \geq~&\frac{3 - e}{8e} \cdot \frac{1}{n} \cdot \frac{n}{2} \\
        =~&\frac{3 - e}{16e} > 0.
    \end{align*}
    
    \Cref{lem:concentration} yields the concentration condition with constant $r$, therefore, by \Cref{thm:negative_drift} [\nameref{thm:negative_drift}], there exists a constant $c' > 0$, such that
    \begin{align*}
        \Prob{T' \leq 2^{c' M/({4 r})}} \in 2^{-\Omega(M)}.
    \end{align*}

    As $\frac{1}{2}n \leq M \leq n$, $\Omega(M) = \Omega(n)$. Choosing $c = \frac{1}{4r}c'$ proves the lower bound on $T'$, which also establishes the same lower bound on $T$.
\end{proofE}

The idea for our proof is to only consider the greater of the two parts and bound the time that it takes all but the most frequent color to vanish.

\subsection{Paths}

Let $n \in \N$ and $P_n$ be the path graph with $n$ vertices. We make the following two reasonable assumptions: the probability of jumping from a $k$-coloring for $k > 20$ to a $2$-coloring in a single step is sufficiently small, and, once such a coloring is reached, the colors are balanced in such a way that the third most frequent color appears sufficiently often. We then show that even starting at this point in time, the time to reach an optimal color is exponential with high probability.

\begin{theoremE}[][end, restate, category=unbiased]
\label{thm:ooea_bounded_degree}
    Let $n \in \N$ and consider the \ooea with fitness function \textsc{NumUsedColors} on $P_n$. Assume that at time $t = 0$, $P_n$ is already $k$-colored with $2 < k \leq 20$ and that there are at least three colors which appear on at least $n/200$ vertices each. Let $T$ be the time it takes for $P_n$ to be $2$-colored.

    Then there exists a constant $c > 0$, such that
    \begin{align*}
        \Prob{T \leq 2^{cn}} \in 2^{-\Omega(n)}.
    \end{align*}
\end{theoremE}

\begin{proofE}
    For $P_n$ to be $2$-colored, at least one of the three most frequent colors needs to vanish. We first fix $i \leq 3$ and consider the time it takes for the $i$\textsuperscript{th} most frequent color to be removed, and then union bound over all $i \leq 3$ to derive the final result.

    Let $i \leq 3$ and $(X_t)_{t \in \N}$ denote the number of vertices that are colored in the $i$\textsuperscript{th} most frequent color, which we call red from now on. We now show that $T_i = \min \{ t \in \N \mid X_t \leq 0 \}$ is exponential with high probability. For this, we intend to show that $(X_t)_{t \in \N}$ exhibits negative drift over the interval $[0, n/200]$.
    
    Let $t \in \N$ and assume $0 < X_t < n/200$. Note that by assumption, $X_0 \geq n/200$. Let $A$ be the event that the current mutation does not increase the number of colors, and $B$ the event that in the current mutation exactly one vertex is flipped.

    \newcommand{\gooddrift}{Y_\downarrow}
    \newcommand{\baddrift}{Y_\uparrow}

    Let $\gooddrift$ denote the number of red vertices turning non-red at time $t$, and analogously $\baddrift$ the number of non-red vertices turning red at time $t$. We now bound $\Ex{\gooddrift \bigmid X_t; A}$ from above and $\Ex{\baddrift \bigmid X_t; A}$ from below.

    For $\gooddrift$, note that there are $X_t$ red vertices, each being flipped with probability $\frac{1}{n}$. Conditioning on $A$ is equivalent to a mutation operator which only considers flips to used colors, where the same holds. Disregarding the probability of then turning non-red, this turns into the following inequality:
    \begin{align*}
        \Ex{\gooddrift \bigmid X_t; A} \leq \frac{X_t}{n} \leq \frac{1}{200}.
    \end{align*}

    Let $C_v$ be the event that the vertex $v$ turns red, and let $\gamma(v)$ be $1$ if $v$ is not adjacent to a red vertex and $0$ otherwise. Let $k \leq 20$ be the total number of colors. For $\baddrift$, we get, using the law of total expectation,
    \begin{align*}
        \Ex{\baddrift \bigmid X_t; A}
        \geq~&\Prob{B \bigmid X_t; A} \cdot \Ex{\baddrift \bigmid X_t; A, B} \\
        \geq~&\frac{1}{e} \cdot \sum_{v \text{~not red}} \Prob{C_v \bigmid X_t; A, B} \\
        =~&\frac{1}{e} \cdot \sum_{v \text{~not red}} \frac{1}{n} \cdot \frac{1}{k} \cdot \gamma(v)\\
        \geq~&\frac{1}{enk} \cdot \sum_{v \text{~not red}} \gamma(v) \\
        \intertext{$\sum_{v \text{~not red}} \gamma(v)$ counts the number of vertices that are neither red nor adjacent to a red vertex. In the worst case, the $X_t$ red vertices on the path pairwise have no neighbor in common, so there are at most $2X_t$ vertices adjacent to a red vertex, yielding}
        \geq~&\frac{1}{enk} \left( n - 3X_t \right) \\
        \geq~&\frac{1}{enk} \left( n - 3 \cdot \frac{n}{200} \right) \\
        =~&\frac{1}{ek} \left( 1 - \frac{3}{200} \right) \\
        \geq~&\frac{1 - \frac{3}{200}}{20e} > \frac{1}{100}.
    \end{align*}

    To show that there is negative drift, we have
    \begin{align*}
        \Ex{X_{t + 1} - X_t \bigmid X_t; A} 
        =~&\Ex{\baddrift \bigmid X_t; A} - \Ex{\gooddrift \bigmid X_t; A} \\
        \geq~&\frac{1}{100} - \frac{1}{200} > 0.
    \end{align*}

    \Cref{lem:concentration} yields the concentration condition with constant $r$, so by \Cref{thm:negative_drift} [\nameref{thm:negative_drift}], there exists a constant $c' > 0$ such that
    \begin{align*}
        \Prob{T_i \leq 2^{c' n/(200r)}} \in 2^{-\Omega(n)}.
    \end{align*}
    Choosing $c_i = {c'}/{200r}$ yields the exponential lower bound for~$T_i$.

    We now set $c \coloneq \min \{ c_i \mid i \leq 3 \}$. Using this, we get, for all $i \leq 3$,
    \begin{align*}
        \Prob{T_i \leq 2^{cn}} \leq \Prob{T_i \leq 2^{c_i n}} \in 2^{-\Omega(n)}.
    \end{align*}

    Recall that $T$ is the time it takes for $P_n$ to be $2$-colored. The event $T \leq 2^{cn}$ occurs if there exists a $i \leq 3$ such that $T_i \leq 2^{cn}$, as at least one of the three most frequent colors needs to vanish. By the union bound, we therefore get
    \begin{align*}
        \Prob{T \leq 2^{cn}} \leq \sum_{i \leq 3} \Prob{T_i \leq 2^{cn}} \in 3 \cdot 2^{-\Omega(n)} = 2^{-\Omega(n)}.
    \end{align*}
\end{proofE}

%% file: sections/ooea_bff.tex
\section{\ooea with \textsc{RankedColors}} 
\label{sec:ooea_bff}

In this section, we consider the \ooea with the fitness function $\textsc{RankedColors}$, which provides additional guidance information for minimizing colors.

\subsection{Upper Bound for Complete Bipartite Graphs}

We first consider the complete bipartite graph.

\begin{theorem} \label{the:run time_bff_kmn}
The \ooea using  $\textsc{RankedColors}$ as fitness and starting with each vertex colored in a unique color has an expected run time of $\mathcal{O}(n^{3} \log \log n )$ fitness evaluations on complete bipartite graphs with $n$ vertices.
\end{theorem}
\begin{proof}
To prove this, we first regard the time it takes to remove one color from the graph. 
Let $\col_t$ be the least frequent color whose partition is not yet uniformly colored at time $t$. This is the color we seek to remove from the graph. We introduce a stochastic process $(X_t)_{t \in \N}$ denoting the number of vertices of color $\col_t$ currently in the graph. Therefore, $X_t$ reaching $0$ corresponds to a color being removed from the graph.

\begin{lemmaE}[][restate,end,category=betterunbiased] \label{lem:remove_all_vertices}
Starting from any search point, it takes at most $en^2 + en^2\cdot\log \Ex{X_0}$ fitness evaluations in expectation to remove the least frequent color whose partition is not yet uniformly colored from the graph.
\end{lemmaE}
We prove this using the \emph{Multiplicative Drift Theorem} which can be found as Theorem (2.5) in \cite{kotzing_theory_2024} and follows from \cite{doerr_multiplicative_2012}. Here, we take advantage of the fact that \textsc{RankedColors} ensures that $X_t$ can never increase.
\begin{proofE}

We use Theorem \ref{thm:mult_drift} [\nameref{thm:mult_drift}] to prove this. 
Let $A$ be the event that exactly one vertex is flipped. Then we have
\begin{align*}
    \Prob{A} =~&\binom{n}{1}\cdot \frac{1}{n} \cdot \left(\frac{{n-1}}{n}\right)^{n-1} 
    =~n\cdot \frac{1}{n} \cdot \left(1 - \frac{1}{n}\right)^{n-1},
\end{align*}
as we need to choose one out of $n$ vertices, flip it and leave the remaining $n - 1$ vertices unchanged.

As $\left(1 - \frac{1}{n}\right)^{n-1}$ converges towards $\frac{1}{e}$ from above, $\Prob{A} \geq \frac{1}{e}$.

Second, assuming we flip only one vertex, we have equal opportunity of flipping either of the colors. Therefore, the probability of that vertex being one of the relevant vertices is $\frac{X_t}{n}$. Additionally, such a flip is only accepted, if we flip to a color such that no monochromatic edge is created. As a loose bound, we assume the worst case \textendash\ that we have $n$ colors and only one meets this requirement. This leads to the probability of $X_t$ decreasing by $1$ being greater than or equal to the probability of flipping exactly one vertex with color $\col_t$ to the only possible color:
\begin{align*}
\Prob{X_t - X_{t+1} = 1\mid X_0, \dots, X_t}
\ge \frac{1}{e} \cdot \frac{X_t}{n} \cdot \frac{1}{n}.
\end{align*}

It further holds that
\begin{align*}
\Prob{X_t - X_{t+1} < 0\mid X_0, \dots, X_t} = 0,
\end{align*}

as the fitness function \textsc{RankedColors} will never accept a change that increases $X_t$, as $\col_t$ is the least frequent color whose number of appearances we can still change. Therefore, we get
\begin{align*}
&\Ex{X_t - X_{t+1} \mid X_0, \dots, X_t} \\
\ge~&\Prob{X_t - X_{t+1} = 1\mid X_0, \dots, X_t} \cdot 1 \\
\ge~& \frac{X_t}{en^2}.
\end{align*}

Let $T = \min \{ t \in \N \mid X_t = 0 \}$. Theorem \ref{thm:mult_drift} [\nameref{thm:mult_drift}] yields that
\begin{align*}
\Ex{T} \le~&\frac{1 + \log \Ex{X_0}}{\frac{1}{en^2 }}
=~(1 + \log \Ex{X_0}) \cdot en^2 \\
=~& en^2 + en^2\cdot\log \Ex{X_0}
\end{align*}
which proves the lemma.
\end{proofE}

Now, we prove that the total run time lies in $\mathcal{O}(n^{3} \log \log n)$. We note that we have to remove a total of $n-2$ colors from the graph. While at least $k$ colors are left in the graph, the rarest color is used for at most $\frac{n}{k}$ vertices, so we start with $X_0 \le \frac{n}{{k}}$.

To obtain an upper bound on the run time of the \ooea, we consider the successive removal of all colors of the graph and use Lemma~\ref{lem:remove_all_vertices} to bound the time of each step. Let $T$ be the time it takes to reduce $n$ colors to $2$. We use Jensen's Inequality to prove that $\Ex{T}\le en^3 +en^3 \cdot \log(\sum_{k = 1}^n\frac{1}{k})$. 
\begin{proofE}[end, category=betterunbiased, text proof={Proof of Theorem \ref{the:run time_bff_kmn}}]
We have
\begin{align*}
    \Ex{T} &~\leq \sum_{k = 2}^n \left(en^2 + en^2 \cdot \log \Ex{X_0} \right)
    ~= \sum_{k = 3}^n en^2 + \sum_{k = 3}^n en^2\log \Ex{X_0} \\
    &\leq en^3 + \sum_{k = 3}^n en^2\log \Ex{X_0}
    ~\leq en^3 + \sum_{k = 3}^n en^2 \cdot \log \frac{n}{k} \\
    &= en^3 + en^2 \sum_{k = 3}^n \log \frac{n}{k}
    ~\leq en^3 + en^2\sum_{k = 1}^n \log \frac{n}{k} \\
    &= en^3 + en^2 \cdot n \cdot\frac{1}{n} \sum_{k = 1}^n \log \frac{n}{k}.
\intertext{
As the $\log$ is concave, with Jensen's Inequality it holds that
}
    & \leq en^3 + en^2\cdot n \cdot \log(\frac{1}{n}\sum_{k = 1}^n \frac{n}{k}) \\
    & = en^3 + en^3 \cdot \log(\frac{n}{n}\sum_{k = 1}^n \frac{1}{k}) 
     ~= en^3 +en^3 \cdot \log(\sum_{k = 1}^n\frac{1}{k}).
\end{align*}

\end{proofE}

With $\sum_{k=1}^n\frac{1}{k} \in \mathcal{O}(\log n)$, this proves the theorem. 
\end{proof}

\subsection{Lower Bound for Complete Bipartite Graphs}
\label{sec:11ea_cbg_lb}

We now prove a lower bound for the run time of the \ooea for coloring a complete bipartite graph using the fitness function $\textsc{RankedColors}$ to show that it takes strictly more time in expectation in comparison with our gray-box operator.

\begin{theoremE}[][restate,end,category=betterunbiased]  \label{thm:ooea_bbff_kmn_lower}
    The \ooea using $\textsc{RankedColors}$ as fitness and starting with each vertex colored in a unique color on any connected bipartite graph with $n$ vertices has a probability in $2^{-\Omega(n)}$ of reaching an optimal coloring in at most $n^{2}$ fitness evaluations. 
\end{theoremE}
To prove this, we bound the time it takes to color all vertices in the bigger partition in a single color, which is a lower bound on the time it takes to find an optimal coloring for the whole graph.

\begin{proofE}
    Let $N\le n$ be the greater partition of the $K_{M,N}$, meaning $N \geq M$ and $N\in \Theta(n)$. We denote this partition by $V_N$.

    Let $A_{t, i, v}$ denote the event that a mutation attempted to flip the vertex $v \in V_N$ to the color $i$ at least once by time $t$. We note that, as the probability of choosing $v$ for mutation and the probability of choosing color $i$ to mutate to are both equal to $\frac{1}{n}$, for $t=1$
    \begin{align*}
        \Prob{A_{1, i, v}}\le \frac{1}{n^{2}}.
    \end{align*}
     Since the mutations are independent, the probability of not attempting to flip $v$ to color $i$ within $t$ steps is $(1-\frac{1}{n^{2}})^{t}$ and it thereby holds that 
     \begin{align*}
         \Prob{A_{t, i, v}}\le 1-\left(1-\frac{1}{n^{2}}\right)^{t}.
     \end{align*}

     Let $A_{t,i}$ denote the event that a mutation attempted to flip all vertices in $V_N$ to color $i$ at least once by time $t$. Since all events $A_{t, i, v}$ are independent, it holds that

     \begin{align*}
         \Prob{A_{t,i}} &= \prod_{v\in V_N} \Prob{A_{t, i, v}} 
         \le\prod_{v\in V_N}1-\left(1-\frac{1}{n^{2}}\right)^{t}
         = \left(1-\left(1-\frac{1}{n^{2}}\right)^{t}\right)^{N}. 
     \end{align*}

    We define $A_t$ as the event that after $t$ steps, all vertices in $V_N$ will be colored in any one color, meaning $A_t$ is bounded from above by the union of all events $A_{t, i}$. With $n$ possible colors and using union bound, $$\Prob{A_{t}}\le \sum_{i=1}^{n}\Prob{A_{t,i}}.$$

    Now, let $t=n^{2}$. Then $(1-\frac{1}{n^{2}})^{t}$ converges towards a constant and we can therefore conclude that there exists a constant $c$ with $0<c<1$ such that, for suitably large $N$, it holds that $\Prob{A_{n^{2},i}}\le c^{N}.$

    With this, we get $\Prob{A_{n^{2}}} \le N \cdot  c^{N},$ which lies in $2^{-\Omega(N)}$ and with $N \in \Theta(n)$ also in $2^{-\Omega(n)}$, proving the theorem. 

\end{proofE}

\subsection{Upper Bound for Paths}

We now prove an upper bound for the run time of the \ooea using the fitness function $\textsc{RankedColors}$ on paths.  First, we prove that it takes $\mathcal{O}(n^{2}(\log n)^{2})$ to reduce the number of colors in the path to $3$. Then, we prove that it takes $\mathcal{O}(n^{7})$ to reduce the number of colors to $2$.

\begin{theoremE}[][restate,end,category=betterunbiased] \label{the:run time_bff_kmn2}
In expectation, the \ooea using the fitness function $\textsc{RankedColors}$ and starting with each vertex colored in a unique color requires $\mathcal{O}(n^{2}(\log n)^{2})$ fitness evaluations to reduce the total number of colors used in a path with $n$ vertices to $3$.
\end{theoremE}
Let $k$ denote the number of colors in use in the graph. We first note that for any given vertex of the currently least frequent color, we are always able to perform a flip. This follows from the fact that we have more colors currently present in the graph ($k \ge 4$, so at least $3$ colors in addition to whichever color the vertex is colored in) than possible neighbors (maximum of $2$), meaning that there is always one color within the graph left to flip to without creating an improper coloring. Therefore, the proof for this theorem is very similar to the one given in Theorem \ref{the:run time_bff_kmn}, with the difference in run time caused by the constant degree of vertices in a path.

\begin{proofE}

We first regard the time it takes to remove one color from the graph. 

Let $\col_t$ be the least frequent color. This is the color we seek to remove from the graph. We introduce a stochastic process $(X_t)_{t \in \N}$ denoting the number of vertices of color $\col_t$ in the graph at time $t$. Let $k$ denote the total number of colors in use in the graph.

\begin{lemma} \label{lem:remove_all_vertices2}
In expectation, it takes $(1 + \log \Ex{X_0}) \cdot n^{2} \cdot \frac{e}{k-3}$ steps to remove all vertices of one color from the graph.
\end{lemma}

\begin{proof}
We use Theorem \ref{thm:mult_drift} [\nameref{thm:mult_drift}] to prove this.

Let $A$ be the event that exactly one vertex is flipped. We have
\begin{align*}
    \Prob{A} =~&\binom{n}{1}\cdot \frac{1}{n} \cdot \left(\frac{{n-1}}{n}\right)^{n-1}
    =~n\cdot \frac{1}{n} \cdot \left(1 - \frac{1}{n}\right)^{n-1}
\end{align*}
as we need to choose one out of $n$ vertices, flip it and leave the remaining $n - 1$ vertices unchanged. 
As $\left(1 - \frac{1}{n}\right)^{n-1}$ converges towards $\frac{1}{e}$ from above, $\Prob{A} \geq \frac{1}{e}$.

The probability of that vertex being one of the $X_t$ relevant vertices is $\frac{X_t}{n}$. With the probability of $X_t$ decreasing by $1$ being greater than or equal to the probability of performing a single mutation which decreases $X_t$
\begin{align*}
\Prob{X_t - X_{t+1} = 1\mid X_0, \dots, X_t}
\ge \frac{1}{e} \cdot \frac{X_t}{n} \cdot \frac{k-3}{n},
\end{align*}

where $\frac{k-3}{n}$ denotes the probability of the flip not worsening the fitness (for which we must choose one of the at least $k-3$ colors not currently bordering the vertex we are flipping).

Furthermore, we have
\begin{align*}
\Prob{X_t - X_{t+1} < 0\mid X_0, \dots, X_t} = 0,
\end{align*}
as the fitness function used will never accept a change that increases $X_t$, which is the least frequent color. Therefore, we get
\begin{align*}
&\Ex{X_t - X_{t+1} \mid X_0, \dots, X_t} \\
\ge~&\Prob{X_t - X_{t+1} = 1\mid X_0, \dots, X_t} \cdot 1
\ge~\frac{1}{e} \cdot \frac{X_t}{n} \cdot \frac{k-3}{n}.
\end{align*}

Let $T = \min \{ t \in \N \mid X_t = 0 \}$. Theorem \ref{thm:mult_drift} [\nameref{thm:mult_drift}] yields that
\begin{align*}
\Ex{T} \le~&\frac{1 + \log \Ex{X_0}}{\frac{1}{n \cdot e} \cdot \frac{k-3}{n}} 
=~(1 + \log \Ex{X_0}) \cdot n^{2} \cdot \frac{e}{k-3}
\end{align*}
which proves the lemma. 
\end{proof} 

Now, we prove that the total run time lies in $\mathcal{O}(n^{2}(\log n)^{2})$.

For this, we note that we have to remove a total of $n-3$ colors from the graph. For each number of total colors ${k}$, the rarest color has at most $n/k$ occurrences, leading to $X_0 \le \frac{n}{{k}}$.

To obtain an upper bound on the run time of the \ooea we consider the successive removal of all colors of the graph and use Lemma~\ref{lem:remove_all_vertices2} to bound the time of each step. Let $T$ be the time it takes to reduce $n$ colors to $2$. Since $X_0 \leq \frac{n}{k}$ we also have $\Ex{X_0} \leq \frac{n}{k}$. We therefore get
\begin{align*}
    \Ex{T} &\leq \sum_{k = 4}^n (1 + \log \Ex{X_0}) \cdot n^{2} \cdot \frac{e}{k-3} \\ 
    &= \sum_{k = 4}^n \frac{e}{k-3}n^{2} + \sum_{k = 4}^n \frac{e}{k-3}n^{2}\log \Ex{X_0} \\
    &\le en^2 \sum_{k = 1}^{n} \frac{1}{k}+ \sum_{k = 4}^n \frac{e}{k-3}n^{2} \log \frac{n}{k} \\
    &\le en^2 \sum_{k = 1}^{n}\frac{1}{k} + en^{2} \sum_{k = 1}^n \frac{1}{k}\log \frac{n}{k} \\
    &\le en^2 \sum_{k = 1}^{n}\frac{1}{k} + en^{2} \sum_{k = 1}^n \frac{1}{k}\log n \\
    &= en^2 \sum_{k = 1}^{n}\frac{1}{k} + en^{2} \log n \sum_{k = 1}^n \frac{1}{k} 
\end{align*}

With $\sum_{k=1}^n\frac{1}{k} \in \mathcal{O}(\log n)$, this lies in $\mathcal{O}(n^{2}(\log n)^{2})$, which proves the theorem. 
\end{proofE}

Now, it remains to prove that it takes $\mathcal{O}(n^{7})$ fitness function evaluations to perform the final reduction of the number of colors from $3$ to $2$.

\begin{theorem} \label{the:3_to_2_bff}
    Reducing the total number of colors in use in the path from $3$ to $2$ takes $\mathcal{O}(n^{7})$ fitness evaluations in expectation, using the \ooea and the fitness function $\textsc{RankedColors}$. 
\end{theorem}

\begin{proof}
Using Theorem~\ref{the:run time_bff_kmn2} we assume that we have reached a state with only $3$ colors in use and resume the analysis from here.

Let $\col_t$ again denote the least frequent color at a time~$t$. We introduce a stochastic process which we call $(X_t)_{t \in \N}$ and which denotes the number of vertices of color $\col_t$ for a time $t$. 

We note that reducing $X_t$ is more complicated than previously, as flipping instances of $\col_t$ to a different color is no longer guaranteed to be possible. If a vertex colored in $\col_t$ has two differently-colored neighbors, we cannot perform a flip, as any flip would lead to an improper coloring. Instead, however, we can perform a swap of this color with its neighbor. This is achieved by flipping both the instance of $\col_t$ and its neighbor to the color of the other so that the colors effectively switch places. Using this, we must then move the instance of $\col_t$ to a position from which it can be flipped. The only positions where this is guaranteed to be possible are on either endpoint of the path, as those vertices have only one neighbor (and, therefore, neighboring color) each.  

We now first analyze the amount of time it takes to remove one vertex of color $\col_t$.

{Let $l$ and $r$ be the leftmost and rightmost vertices colored in $\col_t$ respectively. Let $v \in \{l,r\}$, such that the distance to the closest endpoint (meaning the left endpoint for the left vertex and the right endpoint for the rightmost vertex) is minimal. Without loss of generality, assume that this holds for the rightmost vertex, meaning $v=r$. We refer with $p$ to the color particle corresponding to $v$ as it moves through the path. Now, we prove that it takes $\mathcal{O}(n^{6})$ fitness evaluations to move $p$ to the rightmost point of the path, meaning position $n$.}

In the following proofs, we assume that $p$ always has two differently colored neighbors and that a swap is therefore always possible. This assumption is valid, as otherwise we would already have reached a position from which a flip is possible.

Further, we note that while swaps to the right are always possible, other vertices colored in $\col_t$ might block swaps to the left, as such a swap might create monochromatic edges, making the coloring improper. However, as we wish to move to the right endpoint of the path, this only improves our run time, as it blocks undesired swaps. Therefore, it is a valid assumption for an upper bound that all swaps are always possible. We will work under this assumption in the following. 

\begin{lemmaE}[][restate,end,category=betterunbiased] 
In expectation, it takes $\mathcal{O}( n ^{6})$ fitness evaluations to move $p$ to the right endpoint of the path. 
\end{lemmaE}

We use an upper bound for the run time of the \emph{Unbiased Random Walk} with one barrier, which can be found as Theorem (4.2) in \cite{kotzing_theory_2024}, to prove this.
\begin{proofE}
 We assume the worst case of $p$ starting at position $\frac{n}{2}$, which is the furthest distance from the right endpoint we can start, as otherwise, we would have chosen to move to the left instead.

We introduce a random process $(Y_{t'})_{t'\in \N}$, denoting the position of $p$ in the path. Let $T'$ denote the time it takes for $p$ to reach position $n$ (the right endpoint of the path) and let $t' \le T'$.

We now denote with $A$ the event that $p$ and one of its neighbors are swapped and the operation does not worsen the fitness and note that
\begin{align*}
    \Prob{A}\ge\frac{1}{n^{4}} \cdot \left(\frac{{n-1}}{n}\right)^{n-2} \ge \frac{1}{n^{4} \cdot e}.
\end{align*}

This follows from the probability for mutating a vertex being $\frac{1}{n}$, same as the probability for flipping to the correct color, leading to a probability of $\frac{1}{n^{4}}$ of swapping the colors of the two vertices. Further, since we want to ensure that the swap does not worsen the fitness, we want no other vertex to be mutated, which happens with a probability of $\left(\frac{{n-1}}{n}\right)^{n-2}$, which converges from above towards $\frac{1}{e}$.

It holds that the expected value for $Y_{t'+1}-Y_{t'}$ is the probability of taking a step multiplied with the expected direction and size of the step, meaning the probability of it being a positive step of size $i$ minus the probability of it being a negative step of size $i$, both of which have the same probability for all $i\le n$. This leads to
\begin{align*}
&\Ex{Y_{t'+1}-Y_{t'} \mid Y_0, \dots, Y_{t'}} 
=0.
\end{align*}

The variance can be lower bounded by the probability of taking a step of size $1$, meaning swapping $p$ with either of its neighbors. Therefore
\begin{align*}
 \Var{Y_{t'+1}-Y_{t'} \mid Y_0, \dots, Y_{t'}} \ge 2\Prob{A} \ge \frac{2}{n^{4} \cdot e}.
\end{align*}

With this, it follows from Theorem \ref{thm:upper_variable_drift} [\nameref{thm:upper_variable_drift}] that
\begin{align*}
\Ex{T'\mid X_0}\le~&\frac{(n^2-(\frac{n}{2})^2)}{\frac{2}{n^{4} \cdot e}}
\le~\frac{n^{4} \cdot e}{2} \cdot n^{2}=\frac{n^{6} \cdot e}{2},
\end{align*}
which lies in $\mathcal{O}(n^{6})$, proving the lemma.
\end{proofE}

Once $p$ has reached the right endpoint of the path, we might either flip it to a different color (requiring one mutation) or swap it away again (requiring at least two mutations). Since we are more likely to perform one mutation than two, a simple restart argument shows that, in expectation, removing the particle $p$ takes at most $\mathcal{O}(n^{6})$ fitness evaluations in expectation.

We further note that the fitness function $\textsc{RankedColors}$  ensures that $\Prob{X_{t+1}> X_t \mid X_0, \dots, X_t} = 0$, meaning $X_t$ will never increase.

In the worst case, we start with $X_0 = \frac{n}{3}$, meaning that we need to remove the color $\mathcal{O}(n)$ times, each time taking $\mathcal{O}(n^{6})$ fitness evaluations in expectation. This leads to an expected total run time in $\mathcal{O}(n^{7})$.
\end{proof}

\subsection{Lower Bound for Paths}

We now prove a lower bound of the run time on a path with $n$ vertices. 

\begin{theoremE}[][restate,end,category=betterunbiased]  \label{thm:ooea_bff_paths_lowerbound}
    There exists a $3$-coloring of a path of length $n$ such that the \ooea using $\textsc{RankedColors}$ as fitness requires $\Omega (n^{6})$ fitness evaluations to remove the least frequent color from the graph.
\end{theoremE}

To prove this, we take a situation where the number of colors has already been reduced to $3$ and only one instance of the rarest color remains, which is currently in the middle of the path and has two differently-colored neighbors. As in Theorem \ref{the:3_to_2_bff}, this makes flipping the instance to a different color without changing any other colors in the same mutation impossible, as it would either increase the number of colors in the path or create an improper coloring. 

There are two ways to resolve this situation: The first is to re-color all vertices between the instance and one endpoint of the path such that the instance now has two neighbors of the same color. However, doing this in a single mutation is exponentially unlikely in the distance to the endpoint $d$, as it would require flipping all $d$ vertices to a different color (probability $1/n^d$) and choosing the correct color to flip to for each vertex (probability $1/n^d$). Therefore, we can disregard this option, as we start with $d \in \Theta (n)$.

The second option is to move the instance of the rarest color to a point from which removing it by flipping only a single vertex to a different color is possible. This holds for either endpoint of the path (as those are the only positions where we are guaranteed to have only one neighboring color). We can move the instance by flipping both it and its neighbor to the color of the other such that the colors effectively switch places. The probability of such a swap lies in $\Theta(1/n^4)$, as it requires flipping two vertices to a different color (probability $1/n^2$) and choosing the correct color each time (probability $1/n^2$). Note that swapping more than two colors in one mutation grows increasingly more unlikely for more colors \textendash\ even swapping the colors of three specific vertices has a probability of only $\Theta(1/n^6)$.

Using this, we then perform an unbiased random walk (using the fact that swaps into either direction are equally likely) to model the movement of the instance of the rarest color through the path. We derive a lower bound on the expected time it takes for the instance to reach either endpoint of the path, which corresponds to the time to remove it, from a drift theorem for the Unbiased Random Walk given in \cite{kotzing_first-hitting_2019} as Corollary~27. 
\begin{proofE}
    To prove this, assume a situation where the number of colors has already been reduced to $3$. Let $\col_t$ denote the rarest color. We assume that we have only one vertex of color $\col_t$ left in the path, which is in the middle of the path, and that both neighboring vertices of that vertex have different colors.

    First, we make several observations about the implications of our situation:

    \begin{itemize}[leftmargin=*]
        \item The nature of the fitness function $\textsc{RankedColors}$ ensures that no second instance of $\col_t$ can ever appear in the path. This specifically means that when changing colors, no vertex color can ever be changed to $\col_t$.
        \item All vertices except for those at the endpoints of the path and those bordering the instance of $\col_t$ have two vertices of the same color as neighbors. This follows from the coloring being proper. Therefore, and since we cannot flip any vertices to $\col_t$, no mutation where only one vertex is mutated can ever result in a proper coloring.
        \item Assuming the instance of $\col_t$ has a distance greater than $2$ to either endpoint of the path, the only mutation that does not worsen the fitness where $2$ vertices are flipped corresponds to moving the instance of $\col_t$ to a neighboring vertex, which is accomplished by flipping the vertex colored in $\col_t$ to one neighboring color and simultaneously flipping that neighboring color to $\col_t$. All other mutations where $2$ vertices are flipped would yield an improper coloring or worsen the fitness: 
        
            Any flip of another color to $\col_t$ would worsen the fitness and therefore be rejected. Therefore, in any mutation that chooses two vertices, neither of which is initially colored in $\col_t$, the only other option is to create an improper coloring by flipping to a color already neighboring that vertex.
            Swapping the instance of $\col_t$ with anyone but its neighbors also creates an improper coloring, since that instance has two differently-colored neighbors. 
        
        \item The same logic applies to all mutations flipping more vertices, with the only mutation of size $i \in \N_{> 2}$ that leads to a proper coloring corresponding to $i - 1$ successive swaps (again assuming the instance of $\col_t$ has a distance greater than $i$ to either endpoint of the path).
        \item We note that the probability of flipping all vertices between the instance and an endpoint with a distance of $d$ vertices in a single mutation is equal to the probability of performing $d-1$ specific swaps (which would get the vertex to the position one removed from the endpoint), as both require re-coloring $d$ vertices to the correct color. Thus, we can disregard the probability of flipping all vertices between the instance and the endpoint (and, likewise, the probability of removing the instance directly, which would require $d+1$ vertices to be re-colored and corresponds to moving the instance directly to the endpoint of the path), as it does not offer an asymptotical improvement compared to swapping.
        \item Once the instance of $\col_t$ reaches either endpoint of the path, it can be flipped to a different color.
    \end{itemize}
    It follows that the only way to remove that instance of $\col_t$ from the path is to first move it to either endpoint of the path, which can only be accomplished by performing mutations that essentially swap $\col_t$ with its neighbors.

     We introduce a stochastic process $(Y_t)_{t \in \N}$, which denotes the position of that one instance of $\col_t$ in the path, with $Y_0 = \frac{n}{2}$.
    
    Now, we seek to prove the amount of time $T$ it takes for the instance of $\col_t$ to reach either endpoint of the path, meaning for $Y_t$ to reach either $n$ or $0$. For this, we use Theorem \ref{thm:lower_variable_drift} [\nameref{thm:lower_variable_drift}].

    We first note that for any given $1 \le i \le n$, it holds that
    \begin{align*}
    &\Prob{Y_{t+1} - Y_t = i \mid Y_0, \dots, Y_t}\\
    =&\Prob{Y_{t+1} - Y_t = -i \mid Y_0, \dots, Y_t} 
    \le\bigg(\frac{1}{n}\bigg)^{2(i+1)},
    \end{align*}
    which is the probability of mutating $i+1$ specific vertices (probability of $\frac{1}{n}$ for each mutation) and flipping each to the correct color (probability of $\frac{1}{n}$ for each flip), while not taking into account what happens to the other vertices.

    Since $\Prob{Y_{t+1} - Y_t = i \mid Y_0, \dots, Y_t} =\Prob{Y_{t+1} - Y_t = -i \mid Y_0, \dots, Y_t}$, we get
   $\Ex{Y_{t+1} - Y_t = i \mid Y_0, \dots, Y_t}=0.$

    For the variance, it follows that
    \begin{align*}
        \Var{Y_{t+1} - Y_t = i \mid Y_0, \dots, Y_t} \le \sum_{i=1}^{n} \bigg(\frac{1}{n}\bigg)^{2(i+1)} \cdot i^{2}.        
    \end{align*}

    Further, we know that $\Var{Y_{t+1} - Y_t = i \mid Y_0, \dots, Y_t}>0$, since there is a positive probability of successfully performing a swap. 

    It follows from Theorem~\ref{thm:lower_variable_drift} [\nameref{thm:lower_variable_drift}] that
    \begin{align*}
        \Ex{T \mid Y_0} 
        &\ge \frac{\left(n-\frac{n}{2}\right)\left(\frac{n}{2}\right)}{\sum_{i=1}^{n}\left(\frac{1}{n}\right)^{2(i+1)} \cdot i^{2}}
        = \frac{\frac{n^{2}}{4}}{\sum_{i=1}^{n} \left(\frac{1}{n}\right)^{2(i+1)}\cdot i^{2}} \\
        &=\frac{\frac{n^{2}}{4}}{\frac{1}{n^{4}} +\sum_{i=2}^{n} \left(\frac{1}{n}\right)^{2(i+1)}\cdot i^{2}} 
        \ge \frac{\frac{n^{2}}{4}}{\frac{1}{n^{4}} + \sum_{i=2}^{n} \left(\frac{1}{n}\right)^{2(2+1)}\cdot n^{2}} \\
         &= \frac{\frac{n^{2}}{4}}{\frac{1}{n^{4}} + \sum_{i=2}^{n} \left(\frac{1}{n}\right)^{6}} 
        =\frac{\frac{n^{2}}{4}}{\frac{1}{n^{4}} +(n-3)\left(\frac{1}{n}\right)^{6}}.
    \end{align*}

    With the numerator in $\Theta(n^2)$ and the denominator in $\Theta\left(\frac{1}{n^4}\right)$, this lies in $\Omega (n^{6})$, proving the theorem.
\end{proofE}

%% file: sections/graybox.tex
\section{Gray-Box RLS} \label{sec:graybox}

In this section we analyze RLS with the gray-box operator on various graph classes.

\subsection{Upper Bound for Complete Bipartite Graphs} \label{sec:graybox_cbg}



We start our analysis of the run time of the gray-box operator with an upper bound on the run time on complete bipartite graphs. First, we prove an expected upper bound of $\mathcal{O}(n \log n)$ using the fitness function $\textsc{NumUsedColors}$ and then conclude that it holds likewise for $\textsc{RankedColors}$.

\begin{theorem} \label{the:gb_on_knm}
    RLS using the gray-box operator and the fitness function $\textsc{NumUsedColors}$ and starting with each vertex colored in a unique color finds an optimal coloring on the $K_{M,N}$ with $n$ vertices in $\mathcal{O}(n \log n)$ fitness evaluations in expectation.
\end{theorem} 
\begin{proof}
We use $k$ to denote the total number of colors currently in use in the graph. We consider removing one color after another, decreasing ${k}$ on each fitness level $k$ and first examine the time it takes to remove the rarest color that can still be reduced on level $k$, meaning the time to move from one level to the next. We refer to a color as being able to be reduced if it is possible to reduce its number of appearances in the graph with a single flip mutation. 

Let $\col_t$ denote that color, depending on the time $t$, as it can change which color is the rarest. We introduce a stochastic process which we denote by $(X_t)_{t \in \N}$ and which, for any given time at which there are $k$ colors in use in the graph, denotes the number of vertices of color $\col_t$. If the number of colors in use in the coloring drops below $k$, we set $X_t=0$. Therefore, $X_t$ reaching $0$ corresponds to a color being removed from the graph.

Next we bound the time we spend on each level. 

\begin{lemmaE}[][restate,end,category=graybox]  \label{lem:ct_to_0}
$X_t$ will reach $0$ in at most $8 \cdot \frac{n}{{k}}$ steps in expectation. 
\end{lemmaE}

\textEnd[both, category=graybox]{To prove this, we seek to use the \emph{Additive Drift Theorem}, which can be found as Theorem (2.1) in \cite{kotzing_theory_2024} and follows from Section 3 in \cite{kotzing_first-hitting_2019}. First, however, we prove that the probability of reducing $X_t$ is always greater than the probability of increasing it.}
\begin{lemmaE}[][restate,end,category=graybox] \label{lem:ct_reduced}
    For all $t \in \N$, it holds that
    \begin{align*}
        \Prob{X_{t + 1} < X_t \mid X_0, \dots, X_t} - \Prob{X_{t + 1} > X_t \mid X_0, \dots, X_t} \ge \frac{1}{8}.
    \end{align*}
\end{lemmaE} 
\begin{proofE}
    
To prove this, let $k'$ denote the number of colors in use in the same partition as $\col_t$. We introduce two events: The event $A$ denotes another color in the same partition as $\col_t$ being picked when choosing which color to select a vertex from for mutation. This corresponds to mutating a vertex in the same partition as $\col_t$ that is not itself colored in $\col_t$. The event $B$ denotes a vertex not previously colored in $col_t$ being flipped to $\col_t$ during mutation. 

In the following, we will use the definition of rank given in Definition \ref{def:rank} and denote the rank of a color $c$ with $r_c$. 

We then differentiate between two possible cases: First, we will consider a situation where $\col_t$ is the least frequent color currently in use in the graph. In the second case, we will analyze a situation where the opposite partition is already colored in a single color, which appears less frequently than $\col_t$, making $\col_t$ the second least frequent color. We note that any other case (meaning $\col_t$ having a rank greater than $2$) is impossible, as that would require more than two partitions. 
In each of the two cases, we will differentiate again between situations where there are only two colors left in the partition $\col_t$ appears in and situations where there are more than two colors. 

\textbf{Case 1:} $\col_t$ is the least frequent color in total, so $r_{\col_t}=1$.

    Now
    we get
    \begin{align*}
      \Prob{X_t - X_{t+1} >0 \mid X_0, \dots, X_t} = \frac{1}{2^{r_{\col_t}}}=\frac{1}{2}.
    \end{align*}
    
    On the other hand, the probability of $X_t$ being increased is
    \begin{align*}
    \Prob{X_t - X_{t+1} <0 \mid X_0, \dots X_t} 
    = \Prob{A} \cdot \Prob{B \mid A},
    \end{align*}
    as it corresponds to the probability of first choosing a vertex in the same partition as $\col_t$ but not colored in $\col_t$ (event $A$) and then flipping that vertex to $\col_t$ (event $B$).
    
    We now differentiate between two further cases:
    
        \textbf{Case 1a:} There is only one other color in the same partition as $\col_t$, so $k' = 2$.
        
        It follows that $\Prob{A} \le \frac{1}{4}$, as the chance of picking that single other color is at most $\frac{1}{4}$ (since its rank must be at least $2$); and $\Prob{B \mid A} =1$, as the only color left in the graph that we can flip to (which must be a color in the same partition to avoid improper colorings) is $\col_t$. Therefore
        $$
        \Prob{X_t - X_{t+1} < 0\mid X_0, \dots, X_t} \le \frac{1}{4}.
        $$
        
        \textbf{Case 1b:} There are more than $2$ colors in the same partition as $\col_t$, so $k' > 2$.

        Since picking a vertex not colored in $\col_t$ in the same partition as $\col_t$ is a prerequisite for $X_t$ increasing, it follows that $\Prob{A}$ must be smaller than the probability of $X_t$ not decreasing, meaning
        $$
        \Prob{A} < 1- \Prob{X_t - X_{t+1} >0 \mid X_0, \dots, X_t} = \frac{1}{2}.
        $$
        Further, the probability to flip a vertex's color to $\col_t$, given that we already chose a vertex not colored in $\col_t$ in the same partition, equals $\frac{1}{k'-1}$. With $k' > 2$, we get
        \begin{align*}
          \Prob{B \mid A} =\frac{1}{k'-1} \le \frac{1}{2}.  
        \end{align*}
        Combining the two terms yields
        \begin{align*}
            \Prob{X_t - X_{t+1} <0 \mid X_0, \dots, X_t} \le \frac{1}{2} \cdot \frac{1}{2} = \frac{1}{4}.
        \end{align*}

    Therefore, in both sub-cases for case 1, it holds that 
        \begin{align*}
            &\Prob{X_t - X_{t+1} >0 \mid X_0, \dots, X_t} 
            =\frac{1}{2} 
            >\frac{1}{4} \\
            \ge~&\Prob{X_t - X_{t+1} <0 \mid X_0, \dots, X_t}.
        \end{align*}
        and especially
        \begin{align*}
            &\Prob{X_t - X_{t+1} >0 \mid X_0, \dots, X_t} \\
            -~&\Prob{X_t - X_{t+1} <0 \mid X_0, \dots, X_t} \\
            \ge~ &\frac{1}{4}.
        \end{align*}

\textbf{Case 2:} $\col_t$ is not the globally least frequent color, so $r_{\col_t}=2$.
    
    This means that the least frequent color in total lies in the other partition, but can no longer be reduced. This is the case when all vertices in that partition are already colored in the same color. $\col_t$ is therefore the least frequent color that can still be reduced.
    
    Now, we get
    \begin{align*}
        \Prob{X_t - X_{t+1} >0 \mid X_0, \dots, X_t} = \frac{1}{2^{r_{\col_t}}}\ge \frac{1}{4}.
    \end{align*}
    
    Note that the reason for $\frac{1}{4}$ being a lower bound instead of the exact value is that $\col_t$ might become the least frequent color in total, increasing its rank to $1$ and the probability of being chosen for reduction to $\frac{1}{2}$. As this case only improves our chances of moving in the correct direction and has no other effects, it can be disregarded.
    
    On the other hand, the probability of $\col_t$ being increased is, as above
    \begin{align*}
    &\Prob{X_t - X_{t+1} <0 \mid X_0, \dots, X_t} = \Prob{A} \cdot \Prob{B \mid A}.
    \end{align*}
    
    Again, we differentiate between two further cases:

\textbf{Case 2a:} There is only one other color in the same partition as $\col_t$, so $k' = 2$.

        As there is only one other color (whose rank must be $3$, as there are only $3$ colors in total in the graph) left in the same partition as $\col_t$, it follows that
        \begin{align*}
            \Prob{A} = \frac{1}{8}.
        \end{align*}
        Further, as in Case 1a,
        \begin{align*}
            \Prob{B \mid A} =1.
        \end{align*}
        Therefore
        \begin{align*}
            \Prob{X_t - X_{t+1} <0 \mid X_0, \dots, X_t} = \frac{1}{8}.
        \end{align*}

\textbf{Case 2b:} There are more than $2$ colors in the same partition as $\col_t$, so $k' > 2$.
    
        We introduce an event $C$, denoting the color of rank $1$ being chosen for reduction, and analogously, an event $D$ for the color of rank $2$, which is $col_t$, being chosen.
        
        Now, the probability of choosing another color in the same partition as $col_t$ (event $A$) is less-equal the probability of not choosing either the color of rank $1$ (event $C$) or of rank $2$ (event $D$), meaning
        \begin{align*}
        \Prob{A} 
        \le 1 - \Prob{C} - \Prob{D} 
        = 1- \frac{1}{2}- \frac{1}{4} 
        =\frac{1}{4}.
        \end{align*}
        With $\Prob{B \mid A}$ calculated as in case 1 and $k' > 2$, it holds that
        \begin{align*}
            \Prob{B \mid A} =\frac{1}{k'-1} \le \frac{1}{2}.
        \end{align*}
        It follows that
        \begin{align*}
           \Prob{X_t - X_{t+1} <0 \mid X_0, \dots, X_t} \le \frac{1}{4} \cdot \frac{1}{2} = \frac{1}{8}. 
        \end{align*}
    
    This concludes cases 2b.
    Therefore, in both sub-cases for case 2, it holds that 
    \begin{align*}
       &\Prob{X_t - X_{t+1} >0 \mid X_0, \dots, X_t} 
       \ge\frac{1}{4}
       >\frac{1}{8} \\
       \ge~&\Prob{X_t - X_{t+1} <0 \mid X_0, \dots, X_t} 
    \end{align*}
    and especially 
    \begin{align*}
        &\Prob{X_t - X_{t+1} >0 \mid X_0, \dots, X_t} - \Prob{X_t - X_{t+1} <0 \mid X_0, \dots, X_t} \\
        \ge~&\frac{1}{8},
    \end{align*}
which concludes case 2. 
Since $$\Prob{X_t - X_{t+1} >0 \mid X_0, \dots, X_t} - \Prob{X_t - X_{t+1} <0 \mid X_0, \dots, X_t}\ge\frac{1}{8}$$ holds in both cases, the lemma is proven.
\end{proofE}

Now, using this, we prove Lemma \ref{lem:ct_to_0} with the Additive Drift Theorem.
\begin{proofE} [text proof={Proof of Lemma \ref{lem:ct_to_0}}]
Let $T_k$ denote the time when the process $X_t$ reaches 0 on level $k$. 

It holds that for all $t \le T_k$, $X_t \ge 0$, as no negative amount of vertices can be colored in a given color.

Further, it follows from Lemma \ref{lem:ct_reduced} that 
\begin{align*}
    \Ex{X_t - X_{t+1} \mid X_0, \dots, X_t} \ge \frac{1}{8}.
\end{align*}

With this, it follows from Theorem \ref{thm:additive_drift} [\nameref{thm:additive_drift}] that
\begin{align*}
    \Ex{T_k} \le 8 \cdot \Ex{X_0}.
\end{align*}

We note that $X_0 \le \frac{n}{{k}}$, and therefore $\Ex{X_0} \le \frac{n}{{k}}$. Therefore, it holds that 
\begin{align*}
    \Ex{T_k} \le 8 \cdot \frac{n}{{k}},
\end{align*}
which proves the lemma.
\end{proofE}

With this result, we calculate the total run time. 
Across the entire run of the algorithm, ${k}$ decreases from $n$ to $2$ and there are $n-2$ iterations. 
Therefore, it holds that
\begin{align*}
&\sum_{{k}=2}^{n} \Ex{T_k} 
\le\sum_{{k}=2}^{n} 8 \cdot \frac{n}{{k}} 
=8n \cdot \sum_{{k}=2}^{n} \frac{1}{{k}} 
<8n \cdot  \sum_{{k}=1}^{n} \frac{1}{{k}}.
\end{align*}

With $\sum_{{k}=1}^{n} \frac{1}{{k}} \in \Theta (\log n)$, the total run time lies in $\mathcal{O}(n \log n )$.
\end{proof}

\begin{corollary} \label{cor:bgg_kmn_upper_bound}
     RLS using the gray-box operator and fitness function $\textsc{RankedColors}$ and starting with each vertex colored a unique color finds an optimal coloring on $K_{N,M}$ with $n$ vertices in an expected number of fitness evaluations of $\mathcal{O}(n \log n)$.
\end{corollary}
\begin{proof}
    The proof runs analogous to the above, as the fitness function $\textsc{RankedColors}$ only speeds up the run time by forbidding flips from more frequent to rarer colors.
\end{proof}

\subsection{Lower Bound for Complete Bipartite Graphs}

After proving an upper bound for gray-box optimization on complete bipartite graphs, we now prove a lower bound. For this, we assume use of the fitness function $\textsc{RankedColors}$, although the lower bound holds analogously for the fitness function $\textsc{NumUsedColors}$.

\begin{theoremE}[][restate,end,category=graybox]  \label{the:lower_bound_gbkmn}
    RLS using the gray-box operator and starting with each vertex colored in a unique color takes $\Omega (n \log n)$ fitness evaluations in expectation to find a $2$-coloring on a complete bipartite graph with $n$ vertices.
\end{theoremE} 

Assume that we know in advance in which order colors will be removed from the graph, and that we can number them accordingly. Color $0$ is therefore the last and color $n-3$ the first color to be removed from the graph. The two colors remaining in the final graph do not receive numbers.

Now, our goal is to calculate an expected number of re-colorings a given vertex will go through before reaching its final color, which is either of the two colors remaining in the graph. Using certain assumptions, as detailed in the appendix, we arrive at a recursive term  $w$ for calculating the expected number of additional re-colorings we need for a vertex of color $i$ to be re-colored to one of the final colors:
\begin{align*}
    w(0) = 0;~
    w(i) = \frac{i}{i+2} + \frac{1}{i+2} \sum_{j=0}^{i-1} w(j).
\end{align*}

With this, we have a minimum number of re-colorings of $n-2+ \sum_{i=0}^{n-3} w(i)$, given that we start with one vertex in each color. We first prove that for all $i \in [n-2]$, it holds that $w(i)=\sum_{j=0}^{i-1} \frac{1}{j+3}$ and use this to conclude the above sum is bounded from below by $n \log(n)$, proving the theorem.

\begin{proofE}
We first make several assumptions and argue why none of these assumptions violate generality. Recall that we assume an ordering of the colors, with color $0$ being the last color to be removed from the graph and color $n-3$ the first.

\begin{itemize}[leftmargin=*]
    \item For ease of notation, we assume that any vertex can be flipped to any other color currently in use in the graph at any time. While there is always at least one color (namely any color currently in use in the opposite partition) that we cannot flip to, at least one of the partitions must have $\Theta (k)$ colors to choose from for flipping, with $k$ denoting the number of colors currently in use in the graph. Since we are only interested in the asymptotic run time, it has no bearing on the resulting run time to disregard which partition a color appears in.
    \item A vertex colored in color $i$ can only be re-colored to a color $j$ with $j < i$. This follows from the fact that with the fitness function \textsc{RankedColors}, we can only flip to a color of higher rank, meaning that we have at least those colors to choose from at any given time. Here, we assume that the removal times of colors consistently correspond to the ranks of the colors such that the color of rank $1$ will be removed first, making it color $n-3$ and so on. 
\end{itemize}

We now define a recursive term $w$ for calculating the expected number of additional re-colorings we will need for a vertex of color $i$ to be re-colored to one of the final colors.

We define
\begin{align*}
    w(0) = 0;~
    w(i) = \frac{i}{i+2} + \frac{1}{i+2} \sum_{j=0}^{i-1} w(j).
\end{align*}

This is because first, the last color to be removed (color $0$) will require no additional re-colorings, as it can only be flipped to either of the two final colors. 

For the recursive term, we have $i+2$ colors to choose from for flipping, meaning we have a probability of $\frac{2}{i+2}$ of flipping to either of the final colors. In turn, this leads to a probability of $\frac{i}{i+2}$ of flipping to one of the non-final colors, leading to an expected additional number of re-colorings of $\frac{i}{i+2}$. 

Further, we need to account for any additional re-colorings that arise from the new color we flip to. We have a probability of $\frac{1}{i+2}$ of choosing each of the non-final colors. For each, we then also need to add the number of expected re-colorings, meaning $\sum_{j=0}^{i-1} w(j)$.

Using this recursive function and given that we start with one vertex of each color, the total number of re-colorings is given by $$n-2+ \sum_{i=0}^{n-3} w(i).$$

Here, we have $n-2$ re-colorings that are always necessary to re-color all vertices that start out in a non-final color to either of the two final colors. With $\sum_{i=0}^{n-3} w(i)$, we then add the number of additional re-colorings performed for each vertex. Since we start with one vertex of each color, we must add the number of re-colorings for each possible starting color once.

Before further considering that sum, we first bound $w(i)$.

\begin{lemma}  
   For all $i \in [n-2]$, it holds that $w(i)=\sum_{j=0}^{i-1} \frac{1}{j+3}$.
\end{lemma}
\begin{proof}
We prove this by induction over $i$. 
For $i=1$, we have
\begin{align*}
     w(1) = \frac{1}{1+2} + \frac{1}{1+2} \sum_{j=0}^{1-1} w(j) 
     = \frac{1}{3} + \frac{1}{3} w(0) 
     = \frac{1}{3} 
     = \sum_{j=0}^{1-1} \frac{1}{j+3},
\end{align*}
which proves the statement for $i=1$.

Now, assume that $w(i)=\sum_{j=0}^{i-1} \frac{1}{j+3}$ holds. We note that 
\begin{align*}
    w(i) &= \frac{i}{i+2} + \frac{1}{i+2} \sum_{j=0}^{i-1} w(j) 
\end{align*}
can be likewise written as
\begin{align*}
    \sum_{j=0}^{i-1} w(j)  &= (i+2) \cdot \left(w(i)- \frac{i}{i+2}\right). &&(*)
\end{align*}
By definition
\begin{align*}
    w(i+1) &= \frac{i+1}{i+3} + \frac{1}{i+3} \sum_{j=0}^{i} w(j) \\
    &= \frac{i+1}{i+3} + \frac{1}{i+3} \cdot \left(w(i) + \sum_{j=0}^{i-1} w(j)\right).
\end{align*}
Now, using $(*)$, we get
\begin{align*}
    w(i+1) &= \frac{i+1}{i+3} + \frac{1}{i+3} \cdot \left( w(i) + \left((i+2) \cdot \left(w(i)- \frac{i}{i+2} \right) \right) \right) \\
    &= \frac{i+1}{i+3} + \frac{1}{i+3} \cdot \left( w(i) \cdot (i+3) - \frac{i (i+2)}{i+2} \right) \\
    &= \frac{i+1}{i+3} + w(i)  - \frac{i}{i+3} \\
    &= \frac{1}{i+3} + w(i).
\end{align*}
With $w(i)=\sum_{j=0}^{i-1} \frac{1}{j+3}$, this is equal to 
\begin{align*}
    w(i) = \frac{1}{i+3} + \sum_{j=0}^{i-1} \frac{1}{j+3} 
    = \sum_{j=0}^{i} \frac{1}{j+3},
\end{align*}
which proves our statement and thereby the lemma.
\end{proof}

Now, we use this to prove the theorem. 
We note again that the total number of re-colorings is given by
\begin{align*}
    n-2+ \sum_{i=0}^{n-3} w(i),
\end{align*}
as explained above.
It holds that $$w(i)=\sum_{j=0}^{i-1} \frac{1}{j+3}=\sum_{j=3}^{i+2} \frac{1}{j}= \sum_{j=1}^{i+2} \frac{1}{j} - 1 - \frac{1}{2}.$$ Using Equation (1.4.2) from \cite{doerr_probabilistic_2020}, we know that $$\sum_{j=1}^{i+2} \frac{1}{j} > \log(i+2).$$ Therefore, $w(i) > \log(i+2)- 1 - \frac{1}{2}.$
Using this, we get
\begin{align*}
    n-2+  \sum_{i=0}^{n-3} w(i) &> n-2+  \sum_{i=0}^{n-3}\bigg(\log(i+2)- 1 - \frac{1}{2}\bigg) \\
    &= 2n-5+\frac{n-3}{2} \sum_{i=1}^{n-3}\log(i+2).
\end{align*}

With $\sum_{i=1}^{n-3}\log(i+2) \in \Theta (n \log n)$, this lies in $\Omega(n \log n)$, proving the theorem.
\end{proofE}

\subsection{Upper Bound for Paths using \textsc{RankedColors}} \label{sec:graybox_trees}

We now prove an expected run time of the gray-box operator using the fitness function \textsc{RankedColors} in two steps: First, we prove that it takes $\mathcal{O}(n \log n)$ to reduce the number of colors in the path to $3$. Then, we prove that it takes $\mathcal{O}(n^{4})$ to reduce the number of colors from $3$ to $2$. 

\begin{theoremE}[][restate,end,category=graybox]  \label{the:to_deg+1} 
    Let $G$ be a path with $n$ vertices. Then RLS using the gray-box operator and fitness function $\textsc{RankedColors}$ and starting with each vertex colored in a unique color takes $\mathcal{O}(n \log n)$ fitness evaluations in expectation to reduce the number of colors in use in the graph to $3$.
\end{theoremE}
We first note that, for any given vertex of the currently least frequent color, we are always able to perform a flip. This follows from the fact that we have more colors currently present in the graph ($k \ge 4$, so $3$ colors in addition to whichever color the vertex is colored in) than possible neighbors (maximum of $2$), meaning that there is always one color within the graph left to flip to without creating an improper coloring.

Given this fact, the proof for this theorem is very similar to the proof given in Theorem \ref{the:gb_on_knm} for the gray-box operator on complete bipartite graphs, missing only the necessity of differentiating which rank the color we wish to reduce has, since we can always reduce the least frequent color.  

\begin{proofE}

Note that the \textsc{Swap} operator is irrelevant for the analysis for this theorem. The reason for this is that a swap, by definition, does not change the total number of colors in the algorithm nor does it, in this part, influence our ability to perform a flip. Therefore, it does not improve the fitness of the coloring, as the only way to improve the fitness in an already-proper coloring is to change the number of vertices colored in a given color (either removing a color completely or decreasing the number of appearances of a less frequent color). Therefore, we focus on the \textsc{Flip} operator, which is chosen with constant probability. The following analysis assumes that we have already chosen flip as our operator.

Our analysis proceeds in levels, removing $n-3$ colors one after the other. On each level, we seek to remove the least frequent color which still exists in the coloring, meaning the one with the fewest vertices within the path. Let $\col_t$ denote that color at a time $t$. Let $(X_t)_{t \in \N}$ denote, for any given iteration number $t$, the number of vertices of color $\col_t$. Therefore, $X_t$ reaching $0$ corresponds to a color being removed from the graph. Now we bound the time we spend on each level.

\begin{lemma} \label{lem:one_level}
$X_t$ will reach $0$ in at most $2 \cdot \frac{n}{{k}}$ steps in expectation, with ${k}$ denoting the total number of colors in use in the path.
\end{lemma}
\begin{proof}

We use Theorem~\ref{thm:additive_drift} [\nameref{thm:additive_drift}] to prove this.
Let $T_{k}$ denote the time when the process $X_t$ reaches 0, depending on the total number of colors currently in use in the graph $k$. In the following, we will use the definition of rank given in Definition \ref{def:rank} and denote the rank of a color $c$ with $r_c$. 

Trivially, for all $t \le T_{k}$, $X_t \ge 0$.
We note that
\begin{align*}
    \Prob{X_{t+1}< X_t \mid X_0, \dots, X_t} = \frac{1}{2^{r_{\col_t}}}=\frac{1}{2},
\end{align*}
as this is the probability to change one of the vertices colored $\col_t$ to some other used color with the \textsc{Flip} operator.

From our fitness function we get that we never accept a mutation that increases the number of appearances of the least frequent color; thus
\begin{align*}
    \Prob{X_{t+1}> X_t \mid X_0, \dots, X_t} = 0.
\end{align*}

Overall, we get
\begin{align*}
    \Ex{X_t - X_{t+1} \mid X_0, \dots, X_t} \ge \frac{1}{2}.
\end{align*}
With this and using $X_0 \le \frac{n}{{k}}$, it follows from Theorem \ref{thm:additive_drift} [\nameref{thm:additive_drift}] that
\begin{align*}
    \Ex{T_{k}} \le 2 \cdot \Ex{X_0} \le 2 \cdot \frac{n}{{k}}.
\end{align*}
\end{proof}
Now, we use this to calculate the total run time. As we reduce the number of colors in the path, ${k}$ decreases from $n$ to $3$ and there are $n- 3$ iterations. 
Therefore, it holds that 
\begin{align*}
\sum_{{k}= 4}^{n} \Ex{T_{k}} 
\le \sum_{{k}= 4}^{n} 2 \cdot \frac{n}{{k}} 
= 2n \cdot \sum_{{k}= 4}^{n} \frac{1}{{k}} 
< 2n \cdot  \sum_{{k}=1}^{n} \frac{1}{{k}}.
\end{align*}

With $\sum_{{k}=1}^{n} \frac{1}{{k}} \in \Theta (\log n)$ as per Equation (1.4.2) in \cite{doerr_probabilistic_2020}, this lies in $\mathcal{O}(n \log n)$.

Therefore, we get a total run time of $\mathcal{O}(n \log n)$ for reducing the number of colors to $3$, which proves the theorem.
\end{proofE}

Now, what remains is to reduce the number of colors from~$3$ to~$2$, meaning that we start with ${k}=3$.

\begin{theoremE}[][restate,end,category=graybox]  \label{the:3_to_2}
    Reducing the total number of colors in use in the path, from $3$ to $2$ takes $\mathcal{O}(n^{4})$ fitness evaluations in expectation using RLS with the gray-box mutation operator and the fitness function $\textsc{RankedColors}$. 
\end{theoremE}

Let $\col_t$ again denote the least frequent color at a time $t$. We introduce a stochastic process which we denote by $(X_t)_{t \in \N}$ and which, for any given time, denotes the number of vertices of color $\col_t$.

We note that reducing $X_t$ is more complicated than previously, as flipping instances of $\col_t$ to a different color is no longer guaranteed to be possible: If a vertex colored in $\col_t$ has two differently-colored neighbors, we cannot perform a flip. Here, the \textsc{Swap} operator comes into play, as we must then move the instance of $\col_t$ to a position where it can be flipped. The only positions where this is guaranteed to be possible are the two endpoints of the path, as those vertices have only one neighbor (and, therefore, only one neighboring color) each. 

We now first analyze the amount of time it takes to remove one instance of color $\col_t$ and then conclude the run time for removing all instances. As in the proof for Theorem \ref{the:3_to_2_bff}, we regard the rightmost particle (without loss of generality) and use an upper bound for the run time of the \emph{Unbiased Random Walk} with one barrier, which can be found as Theorem (4.2) in \cite{kotzing_theory_2024}, to bound the expected time it takes to move that particle to the right endpoint of the path. The drastic run time improvement compared to Theorem \ref{the:3_to_2_bff} is caused by the fact that swaps are made significantly more likely by the gray-box operator.

\begin{proofE}
{Let $l$ and $r$ be the leftmost and rightmost vertices colored in $\col_t$ respectively. Let $v \in \{l,r\}$, such that the distance to the closest endpoint (meaning the left endpoint for the left vertex and the right endpoint for the rightmost vertex) is minimal. Without loss of generality, assume that this holds for the rightmost vertex, meaning $v=r$. We refer with $p$ to the color particle corresponding to $v$ as it moves through the path. Now, we prove that it takes $\mathcal{O}(X_t \cdot  n^{2})$ to move $p$ to the rightmost point of the path, meaning position $n$.}

In the following proofs, we assume that $p$ always has two differently colored neighbors and that a swap is therefore always possible. This assumption is valid, as otherwise we would already have reached a position from which a flip is possible.

{Further, we note that while swaps to the right are always possible, other vertices colored in $\col_t$ might block swaps to the left, as such a swap might create monochromatic edges, making the coloring improper. However, as we wish to move to the right endpoint of the path, this only improves our run time, as it blocks undesired swaps. Therefore, it is a valid assumption for an upper bound that all swaps are always possible. We will work under this assumption in the following.}

\begin{lemma}
In expectation, it takes $\mathcal{O}(X_t \cdot  n^{2})$ to move $p$ to the right endpoint of the path. 
\end{lemma}
\begin{proof}
We use the Theorem \ref{thm:lower_variable_drift} [\nameref{thm:lower_variable_drift}] and Theorem \ref{thm:upper_variable_drift} [\nameref{thm:upper_variable_drift}] to prove this. For this, we assume the worst case of $p$ starting at position $\frac{n}{2}$, which is the furthest distance from the right endpoint we can start, as otherwise, we would have chosen to move to the left instead..

We then introduce a random process $(Y_{t'})_{t'\in \N}$, denoting the position of $p$ in the path. Let $T'$ denote the time it takes for that process to reach the right endpoint of the path (position $n$) and let $t' \le T'$. 

We note that the \textsc{Swap} operator is chosen with constant probability $\frac{1}{2}$. Color $\col_t$ is also chosen with constant probability $\frac{1}{2}$. The probability of choosing the instance $p$ equals $\frac{1}{X_t}$, bringing the total probability of taking a step to $\frac{1}{4} \cdot \frac{1}{X_t}$.

It holds that the expected value for $Y_{t'+1}-Y_{t'}$ is the probability of taking a step multiplied with the expected direction of the step, meaning the probability of it being a positive step minus the probability of it being a negative step, both of which are the same here. This leads to
\begin{align*}
&\Ex{Y_{t'+1}-Y_{t'} \mid Y_0, \dots, Y_{t'}}
=~\frac{1}{4} \cdot \frac{1}{X_t} \cdot \left(\frac{1}{2} - \frac{1}{2}\right)
=~0.
\end{align*}

The variance equals the probability of a step being made, meaning the probability of choosing the \textsc{Swap} operator on the currently relevant vertex, which is
\begin{align*}
 \Var{Y_{t'+1}-Y_{t'} \mid Y_0, \dots, Y_{t'}} = \frac{1}{4} \cdot \frac{1}{X_t}.
\end{align*}

With this, it follows from Theorem~\ref{thm:upper_variable_drift} [\nameref{thm:upper_variable_drift}] that
\begin{align*}
\Ex{T' \mid X_0}\le~&\frac{n^2 - \Ex{Y_0^2}}{\frac{1}{4} \cdot \frac{1}{X_t}} 
= \frac{n^2-(\frac{n}{2})^2}{\frac{1}{4} \cdot \frac{1}{X_t}} 
\le 4 \cdot X_t \cdot n^2,
\end{align*}
which proves the lemma.
\end{proof}

Once $p$ has reached the right endpoint of the path, there is a probability of $\frac{1}{4} \cdot \frac{1}{X_t}$ of it moving away again (derived from a probability of $\frac{1}{4}$ of choosing $\col_t$ and the \textsc{Swap} operator and a probability of $\frac{1}{X_t}$ of picking $p$ for the swap). The probability of it being chosen by the \textsc{Flip} operator and flipped to another color is likewise $\frac{1}{4} \cdot \frac{1}{X_t}$ (derived from a probability of $\frac{1}{4}$ of choosing $\col_t$ and the \textsc{Flip} operator and a probability of $\frac{1}{X_t}$ of picking $p$ for the flip). Therefore, we have a constant probability of $p$ being flipped to a different color before being moved away again, namely $\frac{1}{2}$. Using a plain restart argument, in expectation the particle $p$ has to travel to to either endpoint of the path at most twice before being flipped to a different color, which is irrelevant for the asymptotic runtime. This means that, in expectation, removing a particle $p$ takes at most $\mathcal{O}(X_t \cdot n^{2})$.

We further note that since \textsc{RankedColors} ensures that $\\\Prob{X_{t+1}> X_t \mid X_0, \dots, X_t}~=~0$, that $X_t$ will never increase.

At worst, we start with $X_0 = \frac{n}{3}$, meaning that we need to remove the color $\mathcal{O}(n)$ times, each time taking $\mathcal{O}(X_t \cdot n^{2})$. This leads to a total run time in $\mathcal{O}(n^{4})$.
\end{proofE}

%% file: appendix.tex
\newpage
\appendix

\section{Preliminaries}

Here, we briefly state theorems from probability and drift theory that are used in this paper.

\subsection{Probability Theory}

We use the following theorems from probability theory, see Theorem 1.10.1 in \cite{doerr_probabilistic_2020}.

\begin{theorem}[Chernoff Bound] \label{thm:chernoff}
    Let $n \in \N$ and $X_1, \dots, X_n$ be independent random variables taking values in $\{ 0, 1 \}$. Let $X$ be their sum and let $\mu = \Ex{X}$ be the expected value of their sum. Then, for all $\delta > 0$,
    \begin{align*}
        \Prob{X \geq (1 + \delta) \mu} \leq \left( \frac{e^\delta}{(1 + \delta)^{1 + \delta}} \right)^\mu.
    \end{align*}
\end{theorem}

We also use the classic Jensen's Inequality, see \cite{jensen_sur_1906}.

\begin{theorem}[Jensen's Inequality, Finite Form] \label{thm:jensen}
    For a real, convex upward function $\varphi$, numbers $x_1,\dots,x_n$ in its domain and positive weights $a_1,\dots,a_n$  it holds that
    \[
    \varphi\Bigg(\frac{\sum a_ix_i}{\sum a_i}\Bigg) \geq \frac{\sum a_i \varphi(x_i)}{\sum a_i}.
    \]
\end{theorem}

\subsection{Drift Theory}

We use the following drift theorems. The first is an additive drift theorem, which can be found as Theorem (2.1) in \cite{kotzing_theory_2024} and follows from Section 3 in \cite{kotzing_first-hitting_2019}.
\begin{theorem}[Additive Drift, Upper Bound] \label{thm:additive_drift}
    Let $(X_t)_{t \in \N}$ be an integrable process over $\R$, and let $T = \min\{t \in \N \mid X_t \leq 0\}$. Furthermore, suppose the following two conditions hold (non-negativity, drift).
    \begin{enumerate}
        \item[(NN)] For all $t \leq T, X_t \geq 0$.
        \item[(D)] There is a $\delta > 0$, such that, for all $t < T$, it holds that $\Ex{X_{t} - X_{t + 1} ~|~ X_0, \dots, X_t} \geq \delta$.
    \end{enumerate}
    Then \[
    \Ex{T} \leq \frac{\Ex{X_0}}{\delta}.
    \]
\end{theorem}

We will also use a multiplicative drift theorem, which can be found as Theorem (2.5) in \cite{kotzing_theory_2024} and follows from \cite{doerr_multiplicative_2012}.

\begin{theorem}[Multiplicative Drift] \label{thm:mult_drift}
  Let $(X_t)_{t \in \N}$ be an integrable process over $\{0,1\} \cup S$ where $S \subset \R_{>1}$, and let  $T = \min\{t \in \N \mid X_t \leq 0\}$. Assume that there is a $\delta \in \R_+$ such that, for all $s \in S \cup \{1\}$ and all $t < T$ it holds that
  \[
  \Ex{X_{t} - X_{t+1} ~|~ X_0,\dots,X_t} \geq \delta X_t.
  \]
  Then
  \[
  \Ex{T} \leq \frac{1 + \ln\Ex{X_0}}{\delta}.
  \]
\end{theorem}

 The following negative drift theorem can be found as Theorem (5.14) in \cite{kotzing_theory_2024} and follows from \cite{oliveto_simplified_2011} and \cite{oliveto_erratum_2012}.
 
\begin{theorem}[Negative Drift] \label{thm:negative_drift}
    Let $(X_t)_{t \in \N}$ be an integrable random process over $\R$. Suppose there is an interval $[a,b] \subseteq \R$, two constants $\delta, \varepsilon > 0$ and, possibly depending on $\ell = b-a$, a function $r(\ell)$ satisfying $1 \leq r(\ell) \in o(\ell / \log \ell)$ such that, for all $t \in \N$, the following conditions (drift, concentration) hold.
    \begin{enumerate}
        \item[(D)] $\Ex{X_{t+1} - X_t \mid X_0, \dots, X_t; a < X_t < b} \geq \delta$.
        \item[(C)] For all $j \in \N$, $\Prob{\lvert X_{t+1} - X_t \rvert \geq j \mid X_0, \dots, X_t; a < X_t} \leq \frac{r(\ell)}{(1+\varepsilon)^{j}}$.
    \end{enumerate}
	Then there is a constant $c > 0$ such that, for $T = \min\{t \in \N \mid X_t \leq a\}$, we have
    $$
    \Prob{T \leq 2^{c \ell/ r(\ell)} \; \Big| \; X_0 \geq b} \in 2^{-\Omega(\ell/ r(\ell))}.
    $$
\end{theorem}

The following theorem on an upper bound for the run time of an unbiased random walk with one barrier can be found as Theorem (4.2) in \cite{kotzing_theory_2024}

\begin{theorem}[Unbiased Random Walk on the Line, One Barrier] \label{thm:upper_variable_drift}
Let $(X_{t)_t\in \N}$ be random variables over $[0, n]$, and let $T = \min \{ t \in \N \mid X_t =n \}$. Suppose that there is $\delta \in \R_+$ such that, for all $t < T$, we have the following conditions (variance, drift).
    \begin{enumerate}
        \item[(Var)] $\Var{X_{t + 1} - X_t \mid X_0, \dots, X_t } \ge \delta$.
        \item[(D)] $\Ex{X_{t + 1} - X_t \mid X_0, \dots, X_t} = 0$.
    \end{enumerate}
Then
    \begin{align*}
        \Ex{T \mid X_0} \le \frac{n^2-\Ex{X_0 ^2}}{\delta}.
    \end{align*}
\end{theorem} 

Further, the following theorem is given in \cite{kotzing_first-hitting_2019} as Corollary~27. 

\begin{theorem}[Martingale Lower Variable Drift]\label{thm:lower_variable_drift}
Let $(X_t)_{t\in \N}$ be random variables over $[\alpha, \beta] \subset \R$, and let $T = \min \{ t \in \N \mid X_t \in \{ \alpha, \beta \} \}$. Suppose that there is $\delta \in \R_+$ such that, for all $t < T$, we have the following conditions (variance, drift).
    \begin{enumerate}
        \item[(Var)] $\Var{X_{t + 1} - X_t \mid X_0, \dots, X_t } \le \delta$.
        \item[(D)] $\Ex{X_{t + 1} - X_t \mid X_0, \dots, X_t} = 0$.
    \end{enumerate}
    Then
    \begin{align*}
        \Ex{T \mid X_0} \ge \frac{(X_0- \alpha)(\beta - X_0)}{\delta}.
    \end{align*}
\end{theorem} 

\section{RLS with Two Colors} \label{sec:appendix}

We first introduce the following variables, function and states that will be used throughout the proof.

\begin{definition}[RLS Variables]\label{def:vars}

Let $K_{M,N}$ be the complete bipartite graph where the two sets of the partition have size $M$ and $N$. Let $V_N, V_M$ be the respective sets. Let $x \in \{0,1\}^{n}$ be a bit string denoting the color of each vertex. For a vertex $v \in V_N \cup V_M$, if $x[v] = 0$, we call the vertex black and if $x[v] = 1$, we call the vertex white.
We define $w_N,b_N$ as the number of white and black vertices of $V_N$. $w_M, b_M$ are defined likewise for $V_M$. Let $d_N = w_N - b_N$ denote the difference of white and black vertices in $V_N$. Analogously, $d_M = w_M - b_M$.
\end{definition}

\begin{definition}[Monochromatic counting function]
Let
  \[h\colon \{0,1\}^n \rightarrow \N, x \mapsto |\{(u,v) \in E ~|~ x[u] = x[v]\}|.\]
\end{definition}

Since RLS cannot use more colors than an optimal coloring, the number of monochromatic edges of a state is the only relevant term of the fitness function \textsc{NumUsedColors}. We thus use this counting function to reason about the fitness of states.

We define states as follows:
\begin{definition}[RLS States] \label{def:states}
\begin{enumerate}
    \item A state is called \emph{negatively unbalanced} when $d_N$ and $d_M$ have the same sign.
    \item A state is called \emph{balanced} when $d_N = 0$ or $d_M = 0$.
    \item A state is called \emph{positively unbalanced} when $d_N$ and $d_M$ have different signs.
    \item A state is called \emph{optimal} when it is positively unbalanced and $|d_N| = N$ and $|d_M| = M$. This is a global optimum as it defines a proper $2$-coloring.
\end{enumerate}
\end{definition}

We start by showing that RLS can only progress upward in the ordering of states by showing that the fitness values of these states are correspondingly ordered. For this we use the following lemma, stating the fitness function in terms of the defining variables of the states.

\begin{lemma} \label{lem:general_mon}
    Let $x \in \{0,1\}^{n}$. Then $h(x) = \frac{NM}{2} + \frac{d_Nd_M}{2}$.
\end{lemma}
\begin{proof}
    Recall that $h$ counts the number of monochromatic edges of a coloring. As the graph is complete, it can be calculated by the number of white vertices in $V_N$ times the number of white vertices in $V_M$ added to the same product for the black vertices. Thus, $h(x) = w_Nw_M + b_Nb_M$.
   
    We first note that for $p \in \{N,M\}$ \begin{align}
        d_p &= w_p - b_p = w_p - (p - w_p) = 2w_p - p; \\
        w_p &= \frac{d_p + p}{2}.
    \end{align}
    Therefore, we get
    \begin{align*}
        h(x) &= w_Nw_M + b_Nb_M \\
        &= w_Nw_M + (N - w_N)(M - w_M) \\
        &= w_Nw_M + NM - w_NM - w_MN + w_Nw_M \\
        &= 2w_Nw_M - w_MN - w_NM + NM \\
        &= w_M(2w_N - N) - w_NM + NM \\
        &= w_Md_N - w_NM + NM & \text{with (1)} \\
        &= \Big(\frac{d_M + M}{2}\Big)d_N - \Big(\frac{d_N + N}{2}\Big)M + NM & \text{with (2)} \\
        &= \frac{d_Nd_M}{2} + \frac{d_NM}{2} - \frac{d_NM}{2} - \frac{NM}{2} + \frac{2NM}{2} \\
        &= \frac{NM}{2} + \frac{d_Nd_M}{2}.
    \end{align*}
\end{proof}
\begin{corollary}\label{cor:state_ordering}
    Let $x,y,z \in \{0,1\}^{n}$. If $x$ is negatively unbalanced, $y$ is balanced and $z$ is positively unbalanced, then
    \[
    h(x) > h(y) > h(z).
    \]
\end{corollary} 
\begin{proof}
    This follows from Definition~\ref{def:states} and Lemma~\ref{lem:general_mon}. For a balanced state, one of $d_N,d_M$ is zero. Thus $h$ is exactly $\frac{NM}{2}$ for such a state. For negatively unbalanced states $d_N,d_M$ have the same sign, which makes their product positive and $h$ strictly bigger than $\frac{NM}{2}$. Similarly, for a positively unbalanced state, the product is negative and $h$ strictly less than $\frac{NM}{2}$.
\end{proof}

\begin{lemma}\label{lem:dp_change}
    For $p \in \{N,M\}$, $d_p$ only increases or decreases by $2$ in a given step of RLS.
\end{lemma}
\begin{proof}
    By Definition \ref{def:vars} we have $d_p = w_p - b_p$. If we were to increase $w_p$ by $1$ then we have \[d_p' = w_p' - b_p' = (w_p + 1) - (b_p - 1) = w_p - b_p + 2 = d_p + 2.\] Similarly, if we were to decrease $w_p$ we have \[
    d_p' = w_p' - b_p' = (w_p - 1) - (b_p + 1) = w_p - b_p - 2 = d_p - 2.
    \]
    We also see that increases in one color in the partition correspond to decreases in the other color.
\end{proof}

\begin{corollary}\label{cor:difference_parity}
    For $p \in \{N,M\}$, $d_p$ is even if and only if $|V_p|$ is even. Further, only a graph with at least one partition of even size can be in a balanced state.
\end{corollary}
\begin{proof}
    We can have $w_p = b_p$, meaning $d_p = 0$, if and only if $|V_p|$ is even. Since, by Lemma \ref{lem:dp_change}, $d_p$ only increases or decreases in steps of $2$, $d_p$ is always even for $|V_p|$ even, and odd for $|V_p|$ odd. The latter statement holds since a balanced state requires $d_p = 0$ in one partition.
\end{proof}

We now prove an expected time for the first state transition.
\begin{lemma}  \label{lem:neg_to_bal}
    If $N$ or $M$ is even, it takes RLS $~\mathcal{O}(n)$ time in expectation to get from a negatively unbalanced to a balanced state. 
\end{lemma} 
\begin{proof}
    We consider a stochastic process $(S_t)_{t \in \N}$ where $S_t \in \{0,1\}^{n}$. Let $S_t$ be negatively unbalanced. We know that either $d_N < 0$ and $d_M < 0$ or $d_N > 0$ and $d_M > 0$. 
    Since the colors white and black behave identically, they are equivalent up to renaming. Thus, we assume, without loss of generality, $d_N> 0$ and $d_M > 0$, meaning $w_N > b_N$ and $w_M > b_M$. We now show that flipping any of the $w_N,w_M$ white bits will improve the fitness and flipping any of the $b_N,b_M$ black bits will worsen it, thus leading to rejection of the resulting state. By Lemma~\ref{lem:general_mon}, we have
    \begin{align*}
        h(S_t) &= \frac{NM}{2} + \frac{d_Nd_M}{2} \\
        \intertext{Now, without loss of generality, we increase $w_N$ by one by flipping a black bit. Using Lemma \ref{lem:dp_change} we have}
        h(S_{t+1}) &= \frac{NM}{2} + \frac{(d_N+2)d_M}{2} > \frac{NM}{2} + \frac{d_Nd_M}{2} = h(S_t).
        \intertext{Now, we increase $b_N$ by one by flipping a white bit. Then}
        h(S_{t+1}) &= \frac{NM}{2} + \frac{(d_N - 2)d_M}{2} < \frac{NM}{2} + \frac{d_Nd_M}{2} = h(S_t).
    \end{align*}

To reach a balanced state, either $d_N$ or $d_M$ has to equal zero, meaning $w_N = b_N$ or $w_M = b_M$.
We consider the partitions independently, as a change in the difference of one partition does not influence the difference of the other partition.
 
We first consider $V_N$. Let $d'_N$ denote the variable $d_N$ for a state where $\frac{d_N}{2}$ vertices of the more frequent color are flipped. Now we have \[
d'_N = \left(w_N - \frac{d_N}{2}\right) - \left(b_N + \frac{d_N}{2}\right) = w_N - b_N - d_N = d_N - d_N = 0
\] 
We consider a process $(D_t)_{t \in \N}$ where $D_t$ denotes $\frac{d_N}{2}$ at time $t$. 
We have \[\Prob{D_t - D_{t + 1} = 1 ~|~ D_t \geq 1} = \frac{w_N}{n}   \geq \frac{\frac{N}{2}}{n} = \frac{N}{2n}.\] 
and thus \[\Ex{D_t - D_{t + 1} ~|~ D_0,\dots, D_t} \geq \frac{N}{2n}.\]
Let $T = \min\{t \in \N ~|~ D_t = 0\}$. It holds that $D_0 \leq \frac{N}{2}$ and thereby $\Ex{D_0} \leq \frac{N}{2}$. 
It further holds that $D_t - D_{t + 1} \leq 1$ since $d_N$ can only decrease by $2$ in one step, and thereby $\frac{d_N}{2}$ can only decrease by $1$. It follows that $D_t \geq 0$ for all $t < T$.

Then by Theorem \ref{thm:additive_drift} [\nameref{thm:additive_drift}] we have \[\Ex{T} \leq \frac{\frac{N}{2}}{\frac{N}{2n}} = n \in \mathcal{O}(n).\]

If we consider this process in $V_M$ it also has expected time $\mathcal{O}(n)$ which can be proven by the same calculations. Since we are interested in the minimum of both hitting times, we have an expected time of $\mathcal{O}(n)$ for either $d_N$ or $d_M$ becoming $0$, which constitutes a balanced state.
\end{proof}

\begin{corollary}\label{cor:neg_to_bal}
    For $N,M$ odd, the expected time it takes RLS to reach a state with better fitness than all negatively unbalanced states is not worse than $\mathcal{O}(n)$.
\end{corollary}
\begin{proof}
   If only one of $N$ or $M$ is even, we can disregard any flips in the odd partition $P$, as this partition cannot have $d_P = 0$ as given by Corollary~\ref{cor:difference_parity}. To reach a balanced state we refer to the analysis in Lemma \ref{lem:neg_to_bal} for a time bound of $\mathcal{O}(n)$.

   If both $N$ and $M$ are odd, we have that $d_N \neq 0$ and $d_M \neq 0$. Therefore, when starting in a negatively unbalanced state, RLS will directly transition to a positively unbalanced state. Consider the process $(D_t)_{t \in \N}$ of Lemma \ref{lem:neg_to_bal} with $\frac{d_N + 1}{2}$ instead of $\frac{d_N}{2}$. Then $D_t = 0$ implies $d_N = -1 < 0$, which constitutes a positively unbalanced state if $d_M > 0$. This new process then yields the same run time result of $\mathcal{O}(n)$ for RLS to reach a positively unbalanced state.
\end{proof}

Now we concern ourselves with balanced states and prove an expected time for the transition from balanced state to a positively unbalanced one.

\begin{lemma}\label{lem:balanced_progress}
    A balanced state can only be left when bits in the balanced partition are flipped.
\end{lemma}
\begin{proof}
    By Lemma \ref{lem:general_mon} for $x \in \{0,1\}^{n}$ we have \[
    h(x) = \frac{NM}{2} + \frac{d_Nd_M}{2}.
    \]
    When $x$ is a balanced state we have $d_N = 0$ or $d_M = 0$. As long as one of those remains equal to zero, flips in the other partition do not change the fitness. Therefore, only flips of RLS which give both $d_N \neq 0$ and $d_M \neq 0$ will change the fitness.
\end{proof}

\begin{lemma} \label{lem:bal_to_pos}
    Starting from a balanced state, RLS takes in expectation $\mathcal{O}(n)$ iterations to enter a positively unbalanced state.
\end{lemma}

\begin{proof}
We model this process with a Markov chain  $(H_t)_{t \in \N}$ with $4$ states. We name these states as follows: $S$, the starting state; $U_N$, a balanced state with $d_N \ne 0$; $U_M$, a balanced state with $d_M \ne 0$; and $G$ for the positively unbalanced state. We have the following state diagram, where in the following we will argue the depicted transition probabilities (self-loop probabilities are suppressed).

\begin{center}
\begin{tikzpicture}[->,>=stealth,shorten >=2pt, line width=0.5pt, node distance=2cm]
\node [circle, draw] (S) {$S$};
\node [circle, draw] (UN) [above right of=S,xshift=1cm]{$U_N$};
\node [circle, draw] (UM) [below right of=S,xshift=1cm] {$U_M$};
\node [circle, draw] (G) [below right of =UN, xshift=1cm] {$G$};
\path (S) edge[bend left = 10] node[above]{$\frac{N}{n}$} (UN);
\path (S) edge[bend right = 10] node[below]{$\frac{M}{n}$} (UM);
\path (UN) edge[bend left = 10] node[above]{$\frac{M}{2n}$} (G);
\path (UM) edge[bend right = 10] node[below]{$\frac{N}{2n}$} (G);
\path (UN) edge[bend left = 10] node[below]{$\leq \!\frac{N}{n}$} (S);
\path (UM) edge[bend right = 10] node[above, xshift=0.2cm,yshift=-0.1cm]{$\leq\!\frac{M}{n}$} (S);
\end{tikzpicture}
\end{center} 
When in state $S$, flipping any bit in $[M]$ will make this partition unbalanced, giving us $\Prob{H_{t + 1} = U_M ~|~ H_t = S} = \frac{M}{n}$ and, analogously, $\Prob{H_{t + 1} = U_N ~|~ H_t = S} = \frac{N}{n}$.

Now we consider the state $U_N$. We want to bound the probability of transitioning from this state to $G$. For $d_N < 0$ , by Lemma \ref{lem:balanced_progress} we have to flip one of the $\frac{M}{2}$ white bits. Likewise, for $d_N > 0$ we have to flip one of the $\frac{M}{2}$ black bits. This gives $\Prob{H_{t+1} = G ~|~ H_t = U_N} = \frac{M}{2n}$ and, analogously, $\Prob{H_{t+1} = G ~|~ H_t = U_M} = \frac{N}{2n}$. While not necessary for our analysis, we note that the probability to revert to $S$ from $U_N$ is at most $N/n$, since we need to flip a bit in $[N]$ in order to balance this partition.

Note that $G$ is positively unbalanced, so the algorithm cannot transition back to another state in the given state diagram.

To derive an expected hitting time for $G$ for this process, we pessimistically ignore the possibility of starting or staying in $U_N$ or $U_M$, but instead either make it in two steps from $S$ to $G$ or revert to $S$.
Now we have that
\begin{align*}
    \Prob{H_{t + 2} = G ~|~ H_t = S} = \frac{N}{n}\cdot \frac{\frac{M}{2}}{n} + \frac{M}{n}\cdot \frac{\frac{N}{2}}{n} = \frac{NM}{n^2}.
\end{align*}
Using a restart argument, the first hitting time $T$ is stochastically dominated by $2$ times a geometric distribution with parameter $NM/n^2$. Therefore, $\Ex{T} \leq \frac{2n^2}{NM} = \frac{2(N+M)^2}{NM}$. Without loss of generality assume that $N < M$. Then $\Ex{T} \leq \frac{8M^2}{M} \in \mathcal{O}(M) = \mathcal{O}(n)$.
\end{proof}

\begin{corollary}\label{cor:bal_to_pos}
    For only one partition $P \in \{N,M\}$ even, the time to reach a positively unbalanced from a balanced state is $\mathcal{O}(n)$.
\end{corollary}
\begin{proof}
    If only one partition is even, according to Lemma~\ref{lem:balanced_progress}, only flips in this partition can cause a state change. This greatly simplifies the modeling of the process. By observations similar as in Lemma~\ref{lem:bal_to_pos}, we note that in every iteration, a flip of either one of the $\frac{P}{2}$ black or $\frac{P}{2}$ white bits will lead to a positively unbalanced state. The hitting time $T$ of this process is therefore again geometrically distributed with success probability $\frac{P}{2n}$ and $\Ex{T} = \frac{2n}{P} < 2n \in \mathcal{O}(n)$.
\end{proof}

We now prove an expected time for the transition into the optimal state, i.e. the time to reach a proper and optimal coloring.

\begin{lemma}\label{lem:pos_to_opt}
    From a positively unbalanced state, it takes RLS $\mathcal{O}(n \log n)$ iterations in expectation to find an optimal coloring.
\end{lemma} 

\begin{proof}
Assume, without loss of generality, that $d_N < 0$ and $d_M > 0$. The optimal, proper coloring has $d_N = -N$ and $d_M = M$, as the complementary coloring would have to go through a negatively unbalanced or balanced state. According to Corollary \ref{cor:state_ordering}, this would result in a worse fitness and is therefore not possible.

We first want to show that only flips that decrease $d_N$ or increase $d_M$ will be accepted by RLS.
Let $S_t$ be the state at time $t$. The fitness of this state is 
\begin{align*}
    h(S_t) &= \frac{NM}{2} + \frac{d_Nd_M}{2}.
\end{align*}
Similarly to the proof of Lemma \ref{lem:neg_to_bal}, we consider changes in $d_N$ and $d_M$.
As $d_N - 2 < d_N < d_N + 2 \leq 1$ and $-1 \leq d_M - 2 < d_M < d_M + 2$, we see that only decreases in $d_N$ and increases in $d_M$ yield a better fitness. This also shows that we can flip the missing bits in arbitrary order, as only the changes in $d_N, d_M$ matter.

We therefore have to flip the remaining $w_N + b_M$ bits. 
We consider the process $(F_t)_{t \in \N}$ where $F_t$ denotes then number of bits that still have to be flipped at time $t$. 
Then $\Ex{F_{t + 1} - F_t ~|~ F_t} = \frac{F_t}{n}$.
Let $T = \min\{t \in \N ~|~ F_t = 0\}$. As $F_0 \leq \frac{n}{2}$, $\Ex{F_0} \leq \frac{n}{2} < n$ as well.
Then by Theorem \ref{thm:mult_drift} [\nameref{thm:mult_drift}]
\begin{align*}
\Ex{T} \leq~ n(1 + \log \Ex{F_0}) 
\leq~n(1 + \log n) 
\in~ \mathcal{O}(n\log n).
\end{align*}
\end{proof}

We now use these to prove our main theorem of this section.
\begin{proof}[Proof of Theorem \ref{thm:rls_runtime}]
    For graphs with both partitions having even size, by Lemmas \ref{lem:neg_to_bal}, \ref{lem:bal_to_pos} and \ref{lem:pos_to_opt} we have a run time of
    \begin{align*}
        \mathcal{O}(n) + \mathcal{O}(n) + \mathcal{O}(n \log n)
        = \mathcal{O}(n \log n).
    \end{align*}

    For graphs with at least one partition of even size, by Corollaries~\ref{cor:neg_to_bal} and~\ref{cor:bal_to_pos}, as well as Lemma~\ref{lem:pos_to_opt}
    we have a run time of 
    \begin{align*}
        \mathcal{O}(n) + \mathcal{O}(n) + \mathcal{O}(n \log n)
        = \mathcal{O}(n \log n).
    \end{align*}

    Finally, for graphs where both partitions are of odd size, by Corollary~\ref{cor:neg_to_bal} and Lemma~\ref{lem:pos_to_opt}, we have a run time of
    \begin{align*}
        \mathcal{O}(n)  + \mathcal{O}(n \log n) 
        = \mathcal{O}(n \log n).
    \end{align*}
\end{proof}

\section{\ooea with \textsc{NumUsedColors}}

\begin{lemma}[Concentration] \label{lem:concentration}
    Let $G = (V, E)$ be a graph with $n$ vertices and let $C$ be a fixed set of colors. Let $(X_t)_{t \in \N}$ denote the amount of vertices colored in any color from $C$ for the \ooea using fitness function $\textsc{NumUsedColors}$. Then there exist constants $r \geq 1$ and $\varepsilon > 0$ such that, for all $t, j \in \N$, if $0 < X_t$, $$\Prob{\lvert X_{t + 1} - X_t \rvert \geq j \mid X_t} \leq \frac{r}{(1 + \varepsilon)^j}.$$
\end{lemma}

\begin{proof}
    Let $\varepsilon = \frac{3}{e} - 1 > 0$, $r = 2$ and $j \in \N$. For $j = 0$ and $j = 1$, the condition is trivial. For $j \geq 2$, let $Z_1, \dots, Z_n \sim \Ber\left(\frac{1}{n}\right)$ be random variables, such that for each vertex $v \in [n]$, $Z_v = 1$ if that vertex's color is flipped in this mutation and $Z_v = 0$ otherwise. Let $Z = \sum_{v = 1}^n Z_v$ denote the total number of vertices whose color is flipped. Then $Z \sim \mathrm{Bin}\left(n, \frac{1}{n}\right)$, so $\Ex{Z} = 1$.
    
    If the event $\lvert X_{t + 1} - X_t \rvert \geq j$ occurs, then at least $j$ vertices have had their color flipped; that is, $Z \geq j$. In other words,
    \begin{align*}
        \Prob{\lvert X_{t + 1} - X_t \rvert \geq j} \leq \Prob{Z \geq j}.
    \end{align*}

    According to the Chernoff bound (using $\delta = j - 1$),
    \begin{align*}
        \Prob{Z \geq j} \leq \frac{e^{j - 1}}{j^j} = \frac{1}{e} \left( \frac{e}{j} \right)^j.
    \end{align*}

    If $j = 2$, we have
    \begin{align*}
        \frac{1}{e} \left( \frac{e}{2} \right)^2 = \frac{e}{4} = \frac{e^2}{4e} \leq \frac{e^2}{9} = \frac{1}{(1 + \varepsilon)^2} \leq \frac{r}{(1 + \varepsilon)^2}.
    \end{align*}

    If on the other hand $j \geq 3$, we have
    \begin{align*}
        \frac{1}{e} \left( \frac{e}{j} \right)^j \leq \frac{1}{e} \left( \frac{e}{3} \right)^j \leq \left( \frac{e}{3} \right)^j = \left( \frac{1}{1 + \varepsilon} \right)^j = \frac{1}{(1 + \varepsilon)^j} \leq \frac{r}{(1 + \varepsilon)^j}.
    \end{align*}

    Either way, we have shown the inequality.
\end{proof}

\printProofs[unbiased]
\section{\ooea with \textsc{RankedColors}}
\printProofs[betterunbiased]
\section{Gray-Box RLS}
\printProofs[graybox]